\documentclass{article}
\usepackage{algorithm}
\usepackage{algpseudocode}
\usepackage{graphicx}
\usepackage{amsmath}
\usepackage{amssymb}%
\usepackage{verbatim}
\usepackage[preprint]{neurips_2026}

\makeatletter
\renewcommand{\@noticestring}{Preprint. Submitted for peer review on May 6, 2026. This version has been revised thanks to initial feedback.}
\makeatother

\usepackage[utf8]{inputenc} %
\usepackage[T1]{fontenc}    %
\usepackage[hidelinks]{hyperref}       %
\usepackage{url}            %
\usepackage{booktabs}       %
\usepackage{amsfonts}       %
\usepackage{nicefrac}       %
\usepackage{microtype}      %
\usepackage{xcolor}         %

\setcitestyle{round, semicolon, aysep={,}}

\usepackage{tikz}
\usetikzlibrary{arrows.meta}
\usepackage{subcaption}
\usepackage{enumitem}
\setlist[itemize]{leftmargin=1.5em}
\usepackage[disable]{todonotes} %
\usepackage{threeparttable}

\usepackage{array}
\usepackage{tabularx}

\RemoveFromHook{package/amsthm/after}[firstaid/aliascounter]
\usepackage{amsthm}
\usepackage{thmtools}
\usepackage{cleveref}
\usepackage{multirow}

\declaretheorem[numberwithin=section]{theorem}
\declaretheorem[sibling=theorem]{lemma}
\declaretheorem[sibling=theorem]{proposition}
\declaretheorem[sibling=theorem]{corollary}
\declaretheorem[style=definition, sibling=theorem]{definition}
\declaretheorem[style=definition, sibling=theorem]{remark}
\declaretheorem[style=definition, sibling=theorem]{example}

\creflabelformat{equation}{#2#1#3}
\crefname{equation}{eq.}{eqs.}
\Crefname{equation}{Equation}{Equations}

\newcommand{\Q}{\mathbb{Q}}
\newcommand{\R}{\mathbb{R}}
\newcommand{\Z}{\mathbb{Z}}
\newcommand{\A}{\mathbb{A}}

\newcommand{\Qp}{\mathbb{Q}_p}
\newcommand{\Zp}{\mathbb{Z}_p}

\newcommand{\mW}{\boldsymbol{W}}
\newcommand{\vx}{\boldsymbol{x}}
\newcommand{\vy}{\boldsymbol{y}}
\newcommand{\ve}{\boldsymbol{e}}
\newcommand{\vb}{\boldsymbol{b}}
\newcommand{\vc}{\boldsymbol{c}}
\newcommand{\vq}{\boldsymbol{q}}
\newcommand{\vr}{\boldsymbol{r}}
\newcommand{\vs}{\boldsymbol{s}}

\newcommand{\vv}{\boldsymbol{v}}
\newcommand{\vw}{\boldsymbol{w}}
\newcommand{\vz}{\boldsymbol{z}}
\newcommand{\vtheta}{\boldsymbol{\theta}}
\newcommand{\vzeta}{\boldsymbol{\zeta}}
\newcommand{\vI}{\boldsymbol{I}}
\newcommand{\vgamma}{\boldsymbol{\gamma}}

\DeclareMathOperator*{\argmax}{argmax}

\newcommand{\nparams}{d}   %
\newcommand{\nin}{m}       %
\newcommand{\nout}{n}      %
\newcommand{\nex}{N}       %

\title{Continuous Optimization for \(p\)-adic Models}

\author{%
  Julian Salazar\\
  Google DeepMind\\
  \texttt{julsal@google.com} \\
  \And
  Dimitri Kanevsky\\
  Google DeepMind\\
  \texttt{dkanevsky@google.com} \\
  \And
  Matt Harvey\\
  Google DeepMind\\
  \texttt{mattharvey@google.com} \\
  \AND
  Pascal Getreuer\\
  Google DeepMind\\
  \And
  Lucas Dixon\\
  Google DeepMind\\
}

\begin{document}

\maketitle

\begin{abstract}
We present the first method for native, continuous gradient descent for machine learning models with $p$-adic parameters. Existing native optimizers are discrete, mostly combinatorial searches, as the $p$-adic numbers $\mathbb{Q}_p$ are totally disconnected, with standard losses that are flat away from their minima. To enable continuous optimization, we propose working with $\mathbb{Q}_p$ via its Berkovich affine line: a canonical, path-connected expansion of $\mathbb{Q}_p$ that preserves its isometries and uniquely extends its analytic maps. This hull is a metric tree with interpretable points and local derivatives, which we show enables effective optimizers and backpropagation. We formulate gradient descent and show that its approximations efficiently learn linear models with coefficients in $\mathbb{Q}_p$ to do modular arithmetic, an XOR-like task not expressible by linear models in $\mathbb{R}$. We also demonstrate momentum and Adam variants, linear regression, and classification on binary-encoded hierarchies (Quillian semantic networks), addressing open problems posed by \citet{martins2025learning}.
\end{abstract}

\section{Introduction}

Most machine learning methods assume inputs and parameters in real Euclidean
space \(\R^n\). The field of \(p\)-adic numbers \(\Qp\) offers another
geometry: its distances \(d(x,y)\) depend on divisibility by \(p\)
and are \textit{ultrametric}, satisfying
\(d(x,z)\le\max\{d(x,y),d(y,z)\}\) (\Cref{sec:background}). For example, there is no real linear classifier that separates even numbers from odd ones.
But in \(\Q_2\), evens are closer to each
other than to any odd number; the $p$-adic classifier \(v_p(wx+b)\ge0\) is perfect at
\(p=2\), \(w=\frac12\), \(b=0\).

Such models may be useful for real-world data with
known ultrametric or \(p\)-adic structure
(e.g., \citealp{mantegna1999hierarchical, dragovich2010genome, hua2021padic}), for hierarchies where properties correlate with subtrees (\Cref{sec:quillian}),
and for approximating number-theoretic algorithms (such as \citealp{satoh2000canonical} or \citealp{kedlaya2001counting}).
While hyperbolic embeddings \citep{nickelKiela2017} and ultrametric
fitting \citep{chierchia2019ultrametric} approximate such geometries,
what they learn is real-valued (embedding coordinates or edge weights of a graph); they do not learn models with \(p\)-adic parameters.
An \textit{isometric embedding}---a map preserving all pairwise
distances---of \(N\) distinct ultrametric points requires at least
\(N-1\) Euclidean dimensions \citep{lemin1985isometric}.

However, direct optimization is difficult: \(\Qp\) is totally disconnected,
and natural losses have vanishing gradients (\Cref{sec:background}).
Existing methods use discrete searches
\citep{khrennikov2000learning, baker2022number, baker2025linear, zubarev2025adic, martins2025learning,
mihara2026padicpolynomial} or real-valued optimization
\citep{zuniga2024deep, nguessan2025vpunns}.
Systematic searches can scale exponentially with dimension or depth;
random walks and annealing avoid exhaustive enumeration but do not
use gradients to guide their proposals.

Instead, we exploit the analytic structure of machine learning models:
affine maps composed with piecewise \textit{analytic functions}, which
admit local Taylor expansions. For \(p\)-adic models, analytic geometry
provides a canonical path-connected enlargement of \(\Qp\): the Berkovich
affine line, to which analytic maps and the \(p\)-adic norm extend
\citep{tate1971rigid, berkovich1990spectral}.
Using these extensions, we construct continuous losses with informative
directional derivatives and demonstrate native
\textit{continuous optimization of models with \(p\)-adic parameters for the first time}.
Specifically:
\begin{itemize}
\item We construct a canonical path-connected space for \(p\)-adic numbers: the tight span \(\Gamma_p\) of \(\Qp\), which is its convex hull in the Berkovich affine line. Model parameters live in the product space \(\Gamma_p^{\nparams}\).
\item We develop the first continuous general \(p\)-adic optimizers: joint gradient descent, Momentum, and Adam, approximated through grouped coordinate descent (\Cref{sec:coordinate-implementation}).
\item We define informative losses that continuously relax \(p\)-adic regression and classification, and efficiently train linear models for modular arithmetic and binary-encoded semantic hierarchies.
\item We formulate backpropagation and release an OSS library (\url{https://github.com/google-deepmind/padic-ml}) to facilitate research towards deep, truly \(p\)-adic neural networks.
\end{itemize}
Our work resolves Open Problem 6.1 (linear multi-class classifiers) and addresses 6.2 (gradient-based optimization) of \citet{martins2025learning}, and opens a path to 6.3 (multi-layer networks).

\section{Background: \(p\)-adics and Related Work}
\label{sec:background}

A \textit{field} is a set where addition, multiplication, and their inverses (subtraction, division)
are well-defined (except \(x /0\)) and satisfy the usual arithmetic laws.
The \textit{\(p\)-adic numbers} \(\Qp\) form such a field \citep{gouvea2020padic} and can be defined as the set of formal series
\[
x := \sum_{i\ge k}a_i p^i,\quad k\in\Z,\quad a_i\in\{0,\dotsc,p-1\},\quad \text{with }a_k\ne0\text{ unless all }a_i=0.
\]
The series with no negative powers of \(p\) form the
\textit{\(p\)-adic integers \(\Zp\)}.
Written as \textit{digits}, the series extends infinitely
\textit{to the left}: \(\dotsc a_2a_1a_0.a_{-1}\dotsc a_k\).
Addition and multiplication are defined by adding and multiplying
the series and carrying in base \(p\). Subtraction \(x-y\) is the unique
\(z\) satisfying \(y+z=x\); for \(y\ne0\), division \(x / y\) is the unique \(z\)
satisfying \(yz=x\).
For example, \(\dotsc666_7=-1\), since adding \(1\) carries forever
to give \(\dotsc000_7\). 

The \textit{valuation} \(v_p(x)\) is \(k\), the index of the rightmost
nonzero digit, with \(v_p(0):=+\infty\). The \textit{\(p\)-adic absolute value} or \textit{norm} is \(|x|_p := p^{-v_p(x)}\), with \(|0|_p := 0\). For example, \(1.25=1\cdot2^0+0\cdot2^{-1}+1\cdot2^{-2}\), so \(v_2(1.25)=-2\) and \(|1.25|_2=4\). The \textit{\(p\)-adic distance} \(d(x,y) = |x - y|_p\) is \(\textit{ultrametric}\), i.e., non-negative, symmetric, zero only when \(x = y\), and satisfies the \textit{strong triangle inequality}:
\begin{equation}
\label{eq:ultrametric}
|x - z|_p \le \max(|x - y|_p, |y - z|_p), \quad \text{with equality when } |x-y|_p \neq |y-z|_p.
\end{equation}
This discretely \(\R\)-valued function has unintuitive behavior: \(d(0,25) = \frac{1}{25}\) but \(d(0,25.1) = 5\).

\(\Qp\) is not an \textit{ordered field}: no total ordering is compatible with its addition and multiplication. Unlike \(p\)-adic analysis, which studies maps \(f : \mathbb{Q}_p^m \to \mathbb{Q}_p^n\), optimization minimizes a loss, requiring values in an ordered field like \(\R\).

\(\Qp\) is \textit{totally path-disconnected}, i.e., all \textit{paths} (continuous functions \(\gamma: [0, 1] \to \Qp\)) are constant. Hence functions \(L : \mathbb{Q}_p^{\nparams} \to \mathbb{R}\) do not support the usual pathwise derivative.
Natural loss functions like absolute error \(\ell_1(\vtheta) := |f(x; \vtheta) - y|_p\) are locally constant (i.e., ``flat'') away
from their minima, which gives vanishing gradients \citep{baker2025linear,zubarev2025adic}. Indeed, two predictions closer to each other than to the target have the same loss: if \(|f(x;\vtheta) - f(x;\vtheta')|_p < |f(x;\vtheta) - y|_p\), then \cref{eq:ultrametric} forces \(\ell_1(\vtheta') = \ell_1(\vtheta)\). Absolute error is
thus locally constant at parameters that are not exact fits.

See \Cref{fig:q3} for a depiction of \(\Q_3\). %

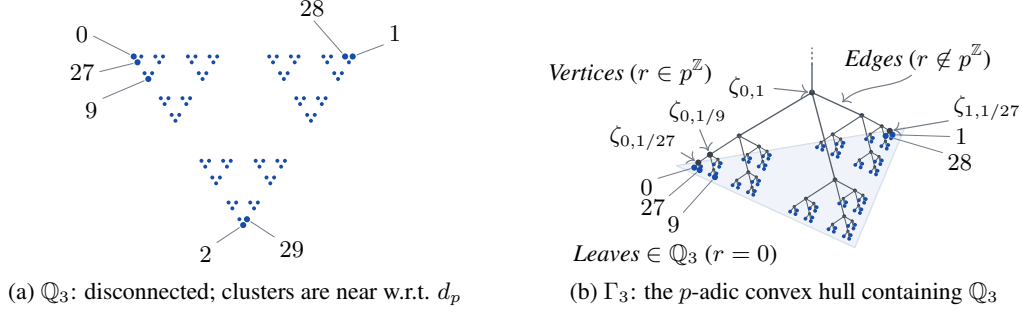
\begin{figure}[t]
\tikzset{every picture/.append style={xscale=0.9, yscale=0.8}} 
\setlength{\abovecaptionskip}{7pt}%
    \centering

    \begin{subfigure}[b]{0.48\textwidth}
        \centering
        \begingroup%
        \definecolor{qPointBlue}{HTML}{164CAA}%
        \definecolor{qLabelGray}{HTML}{737B86}%
        \begin{tikzpicture}[
            x=.46cm,y=.46cm,font=\small,
            qpoint/.style={circle,fill=qPointBlue,inner sep=0pt,minimum size=1.5pt},
            qselected/.style={circle,fill=qPointBlue,inner sep=0pt,minimum size=2.3pt},
            qlabel/.style={inner sep=1pt,outer sep=0pt},
            qleader/.style={draw=qLabelGray,line width=.35pt,shorten <=1pt,shorten >=1.6pt}
        ]
        \foreach \qa/\qDa in {150/0,390/1,270/2} {
            \pgfmathsetmacro{\qAx}{2.5*cos(\qa)}
            \pgfmathsetmacro{\qAy}{2.5*sin(\qa)}
            \foreach \qb/\qDb in {0/0,120/1,-120/2} {
                \pgfmathsetmacro{\qBx}{\qAx+cos(\qa+\qb)}
                \pgfmathsetmacro{\qBy}{\qAy+sin(\qa+\qb)}
                \foreach \qc/\qDc in {0/0,120/1,-120/2} {
                    \pgfmathsetmacro{\qCx}{\qBx+0.4*cos(\qa+\qb+\qc)}
                    \pgfmathsetmacro{\qCy}{\qBy+0.4*sin(\qa+\qb+\qc)}
                    \foreach \qd/\qDd in {0/0,120/1,-120/2} {
                        \pgfmathtruncatemacro{\qIndex}{\qDa+3*\qDb+9*\qDc+27*\qDd}
                        \coordinate (QPoint-\qIndex) at
                            ({\qCx+0.14*cos(\qa+\qb+\qc+\qd)},
                             {\qCy+0.14*sin(\qa+\qb+\qc+\qd)});
                        \node[qpoint] at (QPoint-\qIndex) {};
                    }
                }
            }
        }
        \foreach \qIndex in {0,9,27,1,28,2,29}
            \node[qselected] at (QPoint-\qIndex) {};
        \node[qlabel] (QZero) at (-5.15,2.8) {\(0\)};
        \node[qlabel] (QTwentySeven) at (-5.15,1.45) {\(27\)};
        \node[qlabel] (QNine) at (-4.85,.10) {\(9\)};
        \node[qlabel] (QOne) at (4.85,2.85) {\(1\)};
        \node[qlabel] (QTwentyEight) at (2.20,3.75) {\(28\)};
        \node[qlabel] (QTwo) at (-1.20,-5.10) {\(2\)};
        \node[qlabel] (QTwentyNine) at (1.55,-4.90) {\(29\)};
        \draw[qleader] (QZero.east) -- (QPoint-0);
        \draw[qleader] (QTwentySeven.east) -- (QPoint-27);
        \draw[qleader] (QNine.north east) -- (QPoint-9);
        \draw[qleader] (QOne.south west) -- (QPoint-1);
        \draw[qleader] (QTwentyEight.south) -- (QPoint-28);
        \draw[qleader] (QTwo.north east) -- (QPoint-2);
        \draw[qleader] (QTwentyNine.north west) -- (QPoint-29);
        \end{tikzpicture}%
        \endgroup%

        \caption{\(\mathbb{Q}_3\): disconnected; clusters are near w.r.t. \(d_p\)}
        \label{fig:q3}
    \end{subfigure}
    \hfill
    \begin{subfigure}[b]{0.48\textwidth}
        \centering

        \begingroup%
        \def\qElevation{35}%
        \def\qAzimuth{18}%
        \def\qScale{0.438} %
        \def\qHeightZero{4.05}%
        \def\qHeightOne{1.69}%
        \def\qHeightTwo{0.675}%
        \def\qHeightThree{0.251}%
        \def\qBaseRadius{4.65}%
        \pgfmathsetmacro{\qXX}{\qScale*cos(\qAzimuth)}%
        \pgfmathsetmacro{\qXY}{\qScale*sin(\qAzimuth)*sin(\qElevation)}%
        \pgfmathsetmacro{\qYX}{-\qScale*sin(\qAzimuth)}%
        \pgfmathsetmacro{\qYY}{\qScale*cos(\qAzimuth)*sin(\qElevation)}%
        \pgfmathsetmacro{\qZY}{\qScale*cos(\qElevation)}%
        \definecolor{qLeafBlue}{HTML}{164CAA}%
        \definecolor{qEdgeGray}{HTML}{555E69}%
        \definecolor{qVertexGray}{HTML}{38414A}%
        \definecolor{qPlaneBlue}{HTML}{F0F4FA}%
        \definecolor{qPlaneRim}{HTML}{D7E0ED}%
        \definecolor{qLeaderGray}{HTML}{737B86}%
        \begin{tikzpicture}[
            x={(\qXX cm,\qXY cm)},
            y={(\qYX cm,\qYY cm)},
            z={(0cm,\qZY cm)},
            font=\small, line cap=round, line join=round,
            qedge/.style={draw=qEdgeGray},
            qvertex/.style={circle,fill=qVertexGray,inner sep=0pt,minimum size=1.2pt},
            qleaf/.style={circle,fill=qLeafBlue,inner sep=0pt,minimum size=1.45pt},
            qleader/.style={->,draw=qLeaderGray,line width=.45pt,shorten >=1.5pt},
            qlabel/.style={anchor=west,inner sep=0pt,outer sep=0pt}
        ]
        \path[fill=qPlaneBlue,draw=qPlaneRim,line width=.5pt]
            ({\qBaseRadius*cos(150)},{\qBaseRadius*sin(150)},0) --
            ({\qBaseRadius*cos(30)},{\qBaseRadius*sin(30)},0) --
            (0,-\qBaseRadius,0) -- cycle;
        \coordinate (QRoot) at (0,0,\qHeightZero);
        \foreach \qa/\qDa in {150/0,390/1,270/2} {
            \pgfmathsetmacro{\qAx}{2.5*cos(\qa)}
            \pgfmathsetmacro{\qAy}{2.5*sin(\qa)}
            \coordinate (QA) at (\qAx,\qAy,\qHeightOne);
            \draw[qedge,line width=.52pt] (QRoot) -- (QA);
            \node[qvertex,minimum size=1.8pt] at (QA) {};
            \foreach \qb/\qDb in {0/0,120/1,-120/2} {
                \pgfmathsetmacro{\qBx}{\qAx+cos(\qa+\qb)}
                \pgfmathsetmacro{\qBy}{\qAy+sin(\qa+\qb)}
                \coordinate (QB) at (\qBx,\qBy,\qHeightTwo);
                \pgfmathtruncatemacro{\qPrefixTwo}{\qDa+3*\qDb}
                \coordinate (QBallTwo-\qPrefixTwo) at (QB);
                \draw[qedge,line width=.43pt] (QA) -- (QB);
                \node[qvertex,minimum size=1.5pt] at (QB) {};
                \foreach \qc/\qDc in {0/0,120/1,-120/2} {
                    \pgfmathsetmacro{\qCx}{\qBx+0.4*cos(\qa+\qb+\qc)}
                    \pgfmathsetmacro{\qCy}{\qBy+0.4*sin(\qa+\qb+\qc)}
                    \coordinate (QC) at (\qCx,\qCy,\qHeightThree);
                    \pgfmathtruncatemacro{\qPrefixThree}{\qDa+3*\qDb+9*\qDc}
                    \coordinate (QBallThree-\qPrefixThree) at (QC);
                    \draw[qedge,line width=.36pt] (QB) -- (QC);
                    \node[qvertex] at (QC) {};
                    \foreach \qd/\qDd in {0/0,120/1,-120/2} {
                        \coordinate (QD) at
                            ({\qCx+0.14*cos(\qa+\qb+\qc+\qd)},
                             {\qCy+0.14*sin(\qa+\qb+\qc+\qd)},0);
                        \draw[qedge,line width=.30pt] (QC) -- (QD);
                        \node[qleaf] at (QD) {};
                        \pgfmathtruncatemacro{\qValue}{\qDa+3*\qDb+9*\qDc+27*\qDd}
                        \coordinate (QPoint-\qValue) at (QD);
                    }
                }
            }
        }
        \node[qvertex,minimum size=2.2pt] at (QRoot) {};
        \draw[qedge,line width=.52pt] (QRoot) -- (0,0,{\qHeightZero+1.30});
        \draw[qedge,line width=.52pt,dash pattern=on .57pt off 1.40pt]
            (0,0,{\qHeightZero+1.30}) -- (0,0,{\qHeightZero+2.05});
        \coordinate (QEdgeTarget) at
            ({0.55*2.5*cos(390)},{0.55*2.5*sin(390)},
             {0.45*\qHeightZero+0.55*\qHeightOne});
        \foreach \qValue in {0,9,27,1,28}
            \node[qleaf,minimum size=2.3pt] at (QPoint-\qValue) {};
        \foreach \qNamedPoint in {QBallTwo-0,QBallThree-0,QBallThree-1}
            \node[qvertex,minimum size=2.3pt] at (\qNamedPoint) {};
        \begin{scope}[x={(\qScale cm,0cm)},y={(0cm,\qScale cm)}]
            \node[qlabel,align=left] (QVertexLabel) at (-8.85,4.10)
                {\textit{Vertices} (\(r\in p^{\mathbb Z}\))};
            \node[qlabel] (QEdgeLabel) at (1.10,4.64)
                {\textit{Edges} (\(r\not\in p^{\mathbb Z}\))};
            \node[qlabel] (QLeafLabel) at (-8,-2.8)
                {\textit{Leaves} \(\in\mathbb Q_3\) (\(r=0\))};
            \node[qlabel] (QRootLabel) at (-2.80,3.5) {\(\zeta_{0,1}\)};
            \node[qlabel] (QLcaNine) at (-4.80,2.55) {\(\zeta_{0,1/9}\)};
            \node[qlabel] (QLcaZero) at (-6.80,1.6) {\(\zeta_{0,1/27}\)};
            \node[qlabel] (QZeroLabel) at (-5.75,-.27) {\(0\)};
            \node[qlabel] (QTwentySevenLabel) at (-5.75,-1.00) {\(27\)};
            \node[qlabel] (QNineLabel) at (-4.85,-1.65) {\(9\)};
            \node[qlabel] (QLcaOne) at (4.75,2.74) {\(\zeta_{1,1/27}\)};
            \node[qlabel] (QOneLabel) at (4.80,1.65) {\(1\)};
            \node[qlabel] (QTwentyEightLabel) at (4.55,.70) {\(28\)};
        \end{scope}
        \draw[qleader,shorten <=3pt]
            (QRootLabel.east) -- (QRoot);
        \draw[qleader,shorten <=3pt]
            (QEdgeLabel.south) to[out=220,in=48] (QEdgeTarget);
        \draw[qleader,shorten <=2pt]
            (QLcaNine.south) -- (QBallTwo-0);
        \draw[qleader,shorten <=2pt]
            (QLcaZero.south east) -- (QBallThree-0);
        \draw[qleader,shorten <=2pt]
            (QLcaOne.south west) -- (QBallThree-1);
        \draw[draw=qLeaderGray,line width=.35pt,shorten <=2pt,shorten >=1.5pt]
            (QZeroLabel.north east) -- (QPoint-0);
        \draw[draw=qLeaderGray,line width=.35pt,shorten <=2pt,shorten >=1.5pt]
            (QTwentySevenLabel.north east) -- (QPoint-27);
        \draw[draw=qLeaderGray,line width=.35pt,shorten <=2pt,shorten >=1.5pt]
            (QNineLabel.north east) -- (QPoint-9);
        \draw[draw=qLeaderGray,line width=.35pt,shorten <=2pt,shorten >=1.5pt]
            (QOneLabel.west) -- (QPoint-1);
        \draw[draw=qLeaderGray,line width=.35pt,shorten <=2pt,shorten >=1.5pt]
            (QTwentyEightLabel.north west) -- (QPoint-28);
        \end{tikzpicture}%
        \endgroup%

        \caption{\(\Gamma_3\): the \(p\)-adic convex hull containing \(\mathbb Q_3\)}
        \label{fig:berkovich}
    \end{subfigure}
    \caption{The \(3\)-adics are disconnected (left), but path-connected in the \(p\)-adic hull within the Berkovich affine line (right). Not shown are the infinite set of vertices \(\zeta_{x, p^{-v}}\) from a leaf to e.g. \(\zeta_{0,1}\).}
    \label{fig:q3-berkovich}
    \vspace{0pt}%
\end{figure}

\paragraph{Prior \texorpdfstring{\(p\)}{p}-adic learning.}
Early \(p\)-adic neural networks set weights digit by digit for a single example \citep{albeverio1999padic} or by random search over finite digit prefixes \citep{khrennikov2000learning}.
\citet{bradley2009padic} clusters \(p\)-adic data under a sum-of-norms loss, choosing each cluster center by descending into the most populated branch.
Recent \(p\)-adic optimizers still choose discrete moves: searching finite
digit prefixes \citep{martins2025learning}, modifying coefficient residues
\citep{mihara2026padicpolynomial,mihara2026padiclinear}, or drawing loss-weighted random
increments \citep{zubarev2025adic}.%
Mihara's character-network formulation likewise reduces training to
polynomial feasibility over finite rings \citep{mihara2026padic}.
Other approaches obtain updates through real-valued optimization:
backpropagation in networks organized by non-Archimedean
trees \citep{zuniga2024deep}, or perturbed or rounded
real optimizer steps \citep{nguessan2025vpunns}.
\Cref{app:cost-accounting} describes the recent methods and compares costs in more detail.
In contrast, we operate in an intrinsically defined space $\Gamma_p$ (the tight span of $\Qp$, \Cref{sec:hull}) and show how to 
differentiate 
task losses with respect to its product space in \Cref{sec:optimization}.

\paragraph{Learning on hierarchies and trees.}
Related approaches include hyperbolic embeddings of hierarchies
\citep{nickelKiela2017,chami2020trees} and ultrametric fitting
through real edge weights
\citep{chierchia2019ultrametric,yu2025ultratwd}.
Tropical geometry describes ReLU networks
\citep{zhangNaitzatLim2018} and admits steepest descent in the tropical norm
\citep{talbutMonod2025}; \Cref{sec:forward_propagation} relates both to our radius propagation.
Metric gradient flows and minimizing movements
\citep{mayer1998gradient,ambrosio2008gradient,jordan1998variational},
and subgradient methods on nonpositively curved spaces
\citep{goodwin2026subgradient} have been used to optimize on tree-like spaces like $\Gamma_p$; we discuss these and alternative descent
rules in \Cref{sec:vertex-clipping}.

\section{Analytic Maps on the \(p\)-adic Hull \(\Gamma_p\)}
\label{sec:hull}

A real parameter can move continuously because any two real values are joined by an interval. The field \(\mathbb Q_p\) has no such paths (\Cref{fig:q3}). Hence, we work in the Berkovich affine line \(\mathbb A_{\mathbb Q_p}^{1,\mathrm{an}}\): a canonical, path-connected enlargement of \(\mathbb Q_p\) that gives generalized points over which analytic maps extend uniquely (where convergent) and remain compositional. Appendix~\ref{app:formal-geometry} describes its formal construction as a space of multiplicative seminorms, due to \citet{berkovich1990spectral}. Specifically, within \(\mathbb A_{\mathbb Q_p}^{1,\mathrm{an}}\) we define the \textit{\(p\)-adic hull} \(\Gamma_p\) and show it is the natural connected space for parameters in \(\Qp\).

\subsection{Parameter Space}
\label{sec:parameter-space}

Geometrically, the Berkovich affine ``line'' is an infinite tree, which means every pair of values \(\zeta, \zeta'\in\A_{\Qp}^{1,\mathrm{an}}\) is joined by a unique arc we denote \([\zeta, \zeta']\). For optimization, we restrict individual parameters to the convex hull of $\Qp$ within this tree. For the canonical embedding $\iota : \Qp \hookrightarrow \A_{\Qp}^{1,\mathrm{an}}$, this is
$
  \Gamma_p
  =\bigcup_{x,y\in\mathbb Q_p}[\iota(x),\iota(y)]
$
as a metric space; its infinite subtree is visualized in \Cref{fig:berkovich}.

\paragraph{Points on \(\Gamma_p\).} As we establish
in \Cref{lem:scalar-hull}, we can write \(\Gamma_p\)'s elements as \(\zeta_{x,r}\), indexed by pairs of \(p\)-adic and non-negative real numbers where \textbf{two tuples represent the same point iff they have the same \textit{radius} \(r\) and their \(p\)-adic index $x$ are within this distance}:
\begin{equation}
\label{eq:relation}
\Gamma_p \cong (\Qp \times \R_{\ge0})\ / \left(
\zeta_{x,r} = \zeta_{x',r'} \iff r=r' \ \text{ and } \ |x - x'|_p \le r\right).
\end{equation}
The points in \(\Gamma_p\) corresponding to \(\Qp\) proper are those with zero radius (\(\zeta_{x,0}\)). The radius is the continuous variable \(\Qp\) lacks; a loss that is flat on $x$  (\Cref{sec:background}) can now vary continuously over $r$.

\paragraph{Arcs.} The unique arc \([\zeta_{x,0},\zeta_{y,0}] \subset \Gamma_p\) that connects the hull representatives of \(x,y \in \Qp\) is
\begin{equation}
\label{eq:arc}
\zeta_{x,0} \longrightarrow\text{radius increases}\longrightarrow \zeta_{x,|x-y|_p} = \zeta_{y,|x-y|_p} \longrightarrow\text{radius decreases}\longrightarrow \zeta_{y,0}.
\end{equation}
From \(\zeta_{x,0}\), the radius increases while keeping \(x\) fixed until
\(r=|x-y|_p\), the smallest radius at which \(x\) and \(y\) represent
the same point via the equivalence in \cref{eq:relation}. This point is their
\textit{least common ancestor} (LCA).
Decreasing radius with \(y
\) fixed arrives at $\zeta_{y,0}$
Although radii vary continuously, paths rising from distinct
\(p\)-adic numbers must meet only at radii in \(p^\Z\), because
these are the possible nonzero distances output by \(|x-y|_p\)
(\Cref{sec:background}). These imply an infinite tree
(\Cref{fig:berkovich}):
\begin{itemize}
    \item \textit{\textbf{leaves}}, or \textit{\textbf{ordinary points}}, correspond to \(r=0\), the ordinary \(p\)-adic numbers in \(\Qp\);
    \item \textit{\textbf{open edges}} correspond to \(r > 0\), \(r \not\in p^\mathbb{Z}\) (i.e., the radius is not \(p\) to an integer exponent);
    \item \textit{\textbf{inner vertices}} correspond to \(r \in p^\mathbb{Z}\). It has \(p+1\) outgoing edges: one where the radius increases (moving \textit{up} the tree), and \(p\) where the radius decreases \textit{down} towards the \(\Qp\) points.
\end{itemize}
Choosing one of the \(p\) downward edges at an inner vertex commits to an additional \(p\)-adic digit. For \(p=3\), three numbers sharing the two rightmost digits meet at radius \(3^{-2}\):
\begin{equation}
\label{eq:commitment}
5=\ldots0\mathbf{0}12_3,\quad 14=\ldots0\mathbf{1}12_3,\quad 23=\ldots0\mathbf{2}12_3,\quad \zeta_{5,3^{-2}}=\zeta_{14,3^{-2}}=\zeta_{23,3^{-2}}.
\end{equation}
Below this radius, their paths separate according to the next digit \(0\), \(1\), or \(2\). The upward edge reaches \(r=3^{-1}\), where only the rightmost digit \(2\) remains fixed.

\paragraph{Distances.} Let \(\zeta_{x,R} = \zeta_{y,R}\) be the LCA of \(\zeta_{x,r}, \zeta_{y,s} \in \Gamma_p\). Its radius $R$ equals \(|x - y|_p\) if they are on divergent branches of the tree (\cref{eq:arc}), else one is above the other and $R$ equals \(\max(r,s)\) (the Hsia kernel; \citealp{baker2010potential}).
Therefore, the \textit{arc length} of \([\zeta_{x,r}, \zeta_{y,s}]\) is
\begin{equation}
\label{eq:arc-metric}
d_{\text{arc}}(\zeta_{x,r},\zeta_{y,s})=(R-r)+(R-s)=2\max(|x-y|_p,r,s)-r-s.
\end{equation}
By construction, the metric $d_{\text{arc}}$ on $\Gamma_p$ makes arcs \(
[\zeta_{x,r}, \zeta_{y,s}]
\) the \textit{geodesics} of $\Gamma_p$. Restricted to ordinary points \(\zeta_{x,0}, \zeta_{y,0}\)---those coming from \(x,y \in \Qp\)---this simplifies to \(2|x-y|_p\).  Hence, $(\Qp, |\cdot|_p) \hookrightarrow (\Gamma_p,\frac{1}{2}d_{\text{arc}})$ is an isometric embedding of metric spaces.
In fact $(\Gamma_p,\frac{1}{2}d_{\text{arc}})$ is, up to isometry, the \textit{tight span} (also called the injective or metric envelope) of \((\Qp, |\cdot|_p)\), i.e., it is the ``smallest'' hyperconvex metric space containing \((\Qp, |\cdot|_p)\), making it a canonical choice for $p$-adic optimization. To our knowledge this identification is new; we state and prove it in \Cref{thm:padic-injective-hull}.

\paragraph{Multiple parameters.} Our aggregate parameter space is the Cartesian product
\(\Gamma_p^{\nparams}\) with \textit{parameter states} \(\vtheta := (\theta_1, \dotsc, \theta_{\nparams})\). The induced (\(\ell_2\)) product metric over its components is therefore:
\begin{equation}
\label{eq:product-metric}
d_2(\vtheta,\vtheta') := \sqrt{\sum_{i=1}^\nparams d_{\text{arc}}(\theta_i, \theta_i')^2}.    
\end{equation}
We take the default \(\ell_2\) so that on products of open edges, steepest descent
is the usual negative-gradient update in
arc-length coordinates \citep{boyd2004convex}, and it is understood how to maintain separate adaptive scales
\citep{becigneul2019riemannian}.
The product \(\Gamma_p^d\) contains all classical parameter values but omits mixed seminorms of the full Berkovich affine space, which may offer additional descent directions (\Cref{rem:product-points}); whether they improve optimization remains open.

\subsection{Forward Propagation}
\label{sec:forward_propagation}

We want to optimize parametric models \(\hat{\vy} := f(\vx;\vtheta)\) with
\(\vx \in \mathbb{Q}_p^{\nin}\), \(\vtheta \in \mathbb{Q}_p^{\nparams}\), and
\(\hat{\vy} \in \mathbb{Q}_p^{\nout}\). Fix an input $\vx$ and write \(F := f(\vx;-)\),
decomposed into \textit{stages} \(F=F_K\circ\cdots\circ F_1\) by the structure of the model, with \(F_k:\Qp^{d_{k-1}}\to\Qp^{d_k}\),
where \(d_0=\nparams\) and \(d_K=\nout\). In this notation, a stage ``uses up'' some parameters and propagates intermediate outputs and the remaining parameters; at the end, only final outputs are left.
If each \(F_k\) is analytic---given locally by convergent Taylor expansions---then it has a unique analytification (\Cref{prop:functoriality}), which induces a unique function we write \(\tilde{F}_k:\Gamma_p^{d_{k-1}}\to\Gamma_p^{d_k}\) that extends $F_k$ and has the same Taylor series (\cref{eq:taylor-series}). Hence the stage-by-stage composite \(\tilde{F}_K\circ\cdots\circ \tilde{F}_1\) also extends \(F\), though it does not equal $\tilde{F}$ in general. For \(\vc \in \Q_p^{d_{k-1}}\) and \(\vr \in \R_{\ge 0}^{d_{k-1}}\), write
\(\vzeta_{\vc,\vr} := (\zeta_{c_1, r_1}, \dotsc, \zeta_{c_{d_{k-1}}, r_{d_{k-1}}}) \in \Gamma_p^{d_{k-1}}\). Write \(F_k\) as its coordinate functions \((F_{k,1}, \dotsc, F_{k,d_{k}})\). Let \(a_{\vI}\) be the coefficient of \(\prod_j(z_j-c_j)^{I_j}\) in the Taylor expansion of \(F_{k,i}\) at \(\vc\), where \(\vI=(I_1,\dotsc,I_{d_{k-1}})\) lists nonnegative integer exponents and \(|\vI|:=\sum_j I_j\) (\cref{eq:taylor-series}).
Then \(\tilde{F}_k(\vzeta_{\vc, \vr}) := \vzeta_{\hat{\vc}, \hat{\vr}} \in \Gamma_p^{d_k}\), where \(\hat{\vc} := F_{k}(\vc)\), and for $i = 1, \dotsc, d_k$:
\begin{equation}
\label{eq:pushforward}
\hat{r}_i := \max_{|\vI| \ge 1} |a_{\vI}|_p \prod_{j=1}^{d_{k-1}} r_j^{\,I_j}, \qquad \mathcal{A}_i := \argmax_{|\vI| \ge 1} |a_{\vI}|_p \prod_{j=1}^{d_{k-1}} r_j^{\,I_j}.
\end{equation}
Hence the output radius is a maximum over coefficients scaled by input radius. The multi-indices (equivalently, monomials) $\mathcal{A}_i$ that attain this max are the \textit{active terms} at the $i$-th output of input $\vzeta_{\vc,\vr}$.
For example, if $p=2$ and \(F(\theta_1,\theta_2)\) expanded at $(0,0)$ is \(1 + \theta_1-\frac12\theta_2^2\), then the cases are
\begin{equation}
\tilde{F}(\zeta_{0,r_1}, \zeta_{0,r_2}) = \zeta_{1,\hat{r}} = \zeta_{1,\max(r_1,2r_2^2)},\quad
\begin{array}{@{}ll@{}}
\hat r = r_1>2r_2^2: & \mathcal A=\{(1,0)\} = \{\theta_1\},\\[1pt]
\hat r = 2r_2^2 > r_1: & \mathcal A=\{(0,2)\} = \{\theta_2^2\},\\[1pt]
\hat r = r_1=2r_2^2: & \mathcal A=\{(1,0),(0,2)\} = \{\theta_1, \theta_2^2\}.\\
\end{array}
\end{equation}
The output $\vzeta_{\hat{\vc}, \hat{\vr}}$ is well-defined independent of the input
representative \(\vc\), and the expression for $\hat{r}_i$ comes from the Gauss seminorm of \(F_{k,i} - F_{k,i}(\vc)\), proved in
\Cref{cor:pushforward}.
For convergent analytic maps, \(\mathcal A_i\) is finite whenever \(\hat r_i>0\).

Evaluation of \cref{eq:pushforward} can proceed stage by stage and is exact from $\vc \to \hat\vc$.
As Berkovich points however, this propagation through $\Gamma_p^{d_0}, \Gamma_p^{d_1}, \dotsc$ can be lossy if a parameter reaches a stage along two paths; then, the composed $\hat{\vr}$ is only an upper bound (\Cref{ex:coordinate-marginals})---\(u-u\) should propagate as radius zero, but as independent stage inputs it does not. This is the \textit{dependency problem} of interval arithmetic \citep{dawood2023automatic}.
Given the computed \(p\)-adic coefficients \(a_{\vI}\), taking logarithms in \cref{eq:pushforward} gives a tropical max-plus operation on log-radii. Propagation through this operation is classical: ReLU networks are tropical rational maps \citep{zhangNaitzatLim2018}, whose backward pass follows the attaining terms as Viterbi backtracking does \citep{viterbi1967error}, and \citet{talbutMonod2025} minimize tropical location problems over \(\R\) by steepest descent in the tropical norm.

Merging stages can tighten the radius bound, at the cost of expanding larger analytic expressions.
Finally, a per-example loss \(\ell(\hat y,y)\) comparing an output \(\hat y\in\Qp^{d_K}\) with its target \(y\) must also be extended. The scalar distance \(\ell(\hat{y}, y) = |\hat{y} - y|_p\) has a canonical extension $\frac{1}{2}d_{\text{arc}}$ on $\Gamma_p$ (\cref{eq:arc-metric}; \Cref{thm:padic-injective-hull}). Substituting $\zeta_{\hat{y},\hat{r}}$ for the prediction and $\zeta_{y,0}$ for the reference gives what we call the \textit{direct loss}%
\begin{equation}
\label{eq:direct_loss_def}
\ell_{d}(\zeta_{\hat{y}, \hat{r}}, y) := \frac{1}{2} d_{\text{arc}} (\zeta_{\hat{y}, \hat{r}}, \zeta_{y, 0}) = \max(|\hat{y} - y|_p, \hat{r}) - \frac{\hat{r}}{2},
\end{equation}
a canonical continuous relaxation whose \(-\hat r/2\) term makes it fall as the radius grows even while the maximum is constant (\Cref{fig:descent}).

\section{Continuous Optimization}
\label{sec:optimization}
\label{sec:bgd}

\subsection{Gradient Descent}
\label{sec:gradient_descent}
\label{sec:vertex-clipping}

Continuous gradient descent in a metric space involves proceeding down
the steepest descent direction in the metric tangent cone
\(T_{\vtheta}\Gamma_p^d\) at a parameter state \(\vtheta \in \Gamma_p^d\) for a real-valued duration. This cone consists of velocities $\vv := (\vq, \vs)$, with per-coordinate \textit{directions} $q_i \in \Sigma_{\theta_i} \Gamma_p$ and speeds $s_i \in \R_{\ge 0}$. The available directions $\Sigma_{\theta_i} \Gamma_p$ at $\theta_i$ are \textit{up} $+$, \textit{down} $-$ (on open edges), and \textit{down via} $0,\dotsc,p-1$ (at inner vertices). An available direction $q_i$ at $\theta_i = \zeta_{c,r}$ has a unit-speed path $\gamma_{i, q_i} : [0,\delta_i] \to \Gamma_p$:
\begin{align}
\label{eq:unit-paths}
    \gamma_{i,+}(h) := \zeta_{c,r+h}, \quad \gamma_{i,-}(h) := \zeta_{c,r-h}, \quad \gamma_{i,\rho}(h) := \zeta_{c+\rho p^{n},\, r-h} \text{ for } \rho\in \{0,\dotsc,p-1\},
\end{align}
where $\delta_i$ is the earliest $h$ from which a unique geodesic is no longer determined, i.e., the next vertex: from an inner vertex, $\delta_i = r(p-1)$ for $\gamma_{i,+}$ and $\delta_i = r(1-1/p)$ for each $\gamma_{i,\rho}$.
The $p$ downward paths at inner vertices represent commitment to the digit \(a_n=\rho\) (\cref{eq:commitment}), where \(r=p^{-n}\) and \(c\) is taken with \(a_n=0\); a more formal treatment is in \Cref{def:metric-tangent-cone}.

Altogether, we can define a path $\vgamma_{\vv}(h) := (\gamma_{1, q_1}(s_1h), \dotsc, \gamma_{d,q_d}(s_dh))$ which starts at $\vgamma_{\vv}(0) = \vtheta$. Under our product metric (\cref{eq:product-metric}), this path is unit-speed when $\|\vv\|_2 := \sqrt{\sum s_i^2} = 1$.
The \textit{one-sided directional
derivative} of \(L:\Gamma_p^d\to\R\) with respect to a velocity $\vv \in T_{\vtheta}\Gamma_p^d$ is
\begin{equation}
D_{\vv} L(\vtheta) := \lim_{h \downarrow 0}
\frac{L(\vgamma_{\vv}(h)) - L(\vtheta)}{h}.
\end{equation}
Write \(\ve_{q_i}\) for the unit velocity with \(s_i=1\) and \(s_j=0\) for \(j\neq i\); only \(\theta_i\) moves, and we call \(D_{\ve_{q_i}}L\) the \textit{slope} of \(L\) at \(\theta_i\) in the direction \(q_i\).
To find the steepest joint direction, we take a unit velocity that minimizes this derivative, then
take a step of that size times a \textit{learning rate} $\alpha_t > 0$:
\begin{equation}
\label{eq:gd}
\vv^*\in\text{argmin}_{\|\vv\|_2=1}
 D_{\vv}L(\vtheta_t),
\qquad
\vtheta_{t+1} := \vgamma_{\vv^*}\bigl(-\alpha_t D_{\vv^*} L(\vtheta_t)\bigr).
\end{equation}
This is the metric space analogue of a Euclidean gradient update $\vtheta_{t+1} = \vtheta_{t} - \alpha_t\nabla L (\vtheta_{t})$. \Cref{fig:descent} follows one parameter through such updates.
We cap the step when the first coordinate reaches a vertex, where the tangent cone changes and slopes must be recomputed. This \textit{vertex clipping} extends \(\vgamma_{\vv^*}\) constantly outside its domain \([0,\min_{i:s_i>0}\delta_i/s_i]\) (\cref{eq:unit-paths}). When all directional derivatives are nonnegative, \cref{eq:gd} therefore gives \(\vtheta_{t+1}=\vtheta_t\).

\begin{figure}[t]
\tikzset{every picture/.append style={xscale=0.9, yscale=0.7}}
\centering
\begin{tikzpicture}[x=0.95cm,y=0.8cm,font=\small,
  edge/.style={gray!70,thick},
  faint/.style={gray!35,thick},
  path/.style={blue!70!black,line width=1.5pt,-{Stealth[length=7pt,width=6pt]}},
  pathend/.style={blue!70!black,line width=1.5pt},
  vert/.style={circle,draw=gray!70,fill=white,inner sep=1.1pt},
  tinyv/.style={circle,fill=gray!70,inner sep=0.6pt},
  stepdot/.style={circle,fill=blue!70!black,inner sep=1.4pt},
  num/.style={blue!70!black,inner sep=1pt,font=\scriptsize},
  lab/.style={align=center,inner sep=1pt},
  ax/.style={gray!60},
  axl/.style={gray!80,inner sep=1.5pt}]

\draw[ax,-{Stealth[length=5pt]}] (-0.5,0) -- (-0.5,3.55) node[above,axl] {$r$};
\foreach \y/\lab in {0/$0$,1/$3^{-3}$,2/$2\cdot3^{-3}$,3/$3^{-2}$}
  \draw[ax] (-0.42,\y) -- (-0.58,\y) node[left,axl] {\lab};

\draw[edge,dashed] (3,3) -- (3,3.5);
\foreach \x in {1,3,5} \draw[edge] (3,3) -- (\x,1);
\foreach \x in {1,3,5} {
  \draw[edge] (\x,1) -- (\x-0.45,0.3333) -- (\x-0.60,0.1111) -- (\x-0.65,0.0370) -- (\x-0.675,0);
  \draw[faint] (\x,1) -- (\x,0.3333) (\x,1) -- (\x+0.45,0.3333);
  \draw[faint] (\x-0.45,0.3333) -- (\x-0.45,0.1111) (\x-0.45,0.3333) -- (\x-0.30,0.1111);
  \node[tinyv] at (\x-0.45,0.3333) {};
  \node[tinyv] at (\x-0.60,0.1111) {};
}
\node[vert] at (3,3) {};
\foreach \x in {1,3,5} \node[vert] at (\x,1) {};
\node[below=2pt,lab] at (0.325,0) {$5$\\[-2pt]{\scriptsize$\ldots0\mathbf{0}12_3$}};
\node[below=2pt,lab] at (2.325,0) {$14$\\[-2pt]{\scriptsize$\ldots0\mathbf{1}12_3$}};
\node[below=2pt,lab] at (4.325,0) {$y=23$\\[-2pt]{\scriptsize$\ldots0\mathbf{2}12_3$}};
\node[lab,align=left,anchor=east] (L) at (2.35,3.3) {LCA of $5,14,23$\\[-1pt]{\scriptsize digits $\ldots\mathbf{12}_3$ fixed}};
\draw[ax] (L.east) -- (2.9,3.03);

\draw[path] (3,2) -- (3,2.93);
\draw[path] (3.07,2.93) -- (4.3333,1.6667);
\draw[path] (4.3333,1.6667) -- (4.96,1.04);
\draw[pathend] (5,1) -- (4.55,0.3333) -- (4.40,0.1111) -- (4.35,0.0370) -- (4.325,0);
\node[stepdot] at (3,2) {};
\node[stepdot] at (3,3) {};
\node[stepdot] at (4.3333,1.6667) {};
\node[stepdot] at (5,1) {};
\node[stepdot,inner sep=1.0pt] at (4.55,0.3333) {};
\node[stepdot,inner sep=0.8pt] at (4.40,0.1111) {};
\node[num,left=2pt]  at (3,2) {start};
\node[num,above right=1pt] at (3.08,3.02) {A};
\node[num,right=3pt] at (4.3333,1.72) {B};
\node[num,right=3pt] at (5,1.05) {C};
\node[num,right=3pt] at (4.55,0.36) {D};

\begin{scope}[shift={(7.0,0)},x=1.05cm]
\draw[ax,-{Stealth[length=5pt]}] (0,0) -- (0,3.55);
\draw[ax,-{Stealth[length=5pt]}] (0,0) -- (4.6,0) node[right,axl] {$s$};
\foreach \y/\lab in {1/$3^{-3}$,2/$2\cdot3^{-3}$,3/$3^{-2}$}
  \draw[ax] (0.1,\y) -- (-0.1,\y) node[left,axl] {\lab};
\foreach \x/\lab in {1/$3^{-3}$,2/$2\cdot3^{-3}$,3/$3^{-2}$,4/$4\cdot3^{-3}$}
  \draw[ax] (\x,0.08) -- (\x,-0.08) node[below,axl,font=\scriptsize] {\lab};
\node[axl,anchor=north] at (2.3,-0.55) {distance travelled $s$};
\draw[red!70!black,thick,dashed] (0,3) -- (1,3) -- (1,0.03) -- (4,0.03);
\node[red!70!black,anchor=west,inner sep=2pt] at (1.05,2.75) {$|c-y|_3$};
\draw[blue!70!black,line width=1.3pt] (0,2) -- (4,0);
\node[blue!70!black,anchor=west,inner sep=2pt] at (1.35,1.9) {$\ell_d=\max(|c-y|_3,r)-\tfrac r2$};
\node[stepdot] at (0,2) {};
\node[stepdot] at (1,1.5) {};
\node[stepdot] at (2.3333,0.8333) {};
\node[stepdot] at (3,0.5) {};
\node[stepdot,inner sep=1.0pt] at (3.6667,0.1667) {};
\node[stepdot,inner sep=0.8pt] at (3.8889,0.0556) {};
\node[num,above right=1pt] at (0,2) {start};
\node[num,below left=1pt]  at (1,1.5) {A};
\node[num,below left=1pt]  at (2.3333,0.8333) {B};
\node[num,below left=1pt]  at (3,0.5) {C};
\node[num,above right=1pt] at (3.6667,0.1667) {D};
\end{scope}
\end{tikzpicture}
\caption{Left: one parameter learning \(y=23\) on \(\Gamma_3\) from \(\zeta_{14,\,2\cdot3^{-3}}\) with learning rate \(8\cdot3^{-4}\)
; steps are clipped at vertices (\Cref{sec:vertex-clipping}). The line toward \(y=23\) represents infinitely many clipped updates, with radius tending to zero. Right: along that path, the center's \(p\)-adic error \(|c-y|_3\) changes only when a digit is chosen, while the direct loss \(\ell_d\) falls throughout.}
\label{fig:descent}
\end{figure}
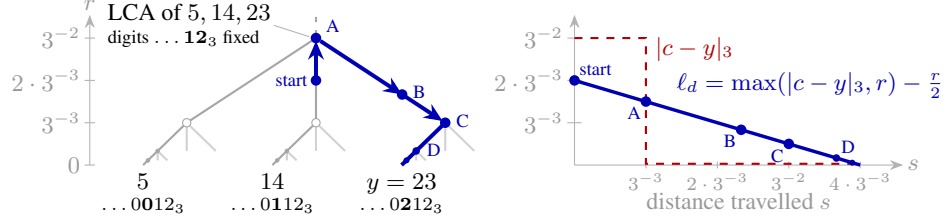

Other conventions are possible. A retraction-based update \citep{boumal2019global} specifies a continuation beyond the vertex. Spending the step's remainder after recomputing slopes at the vertex is a heuristic integration of the metric gradient flow, and still misses slope changes within an edge. That flow can instead be discretized implicitly, minimizing \(L(\vtheta')+d_{\mathrm{arc}}(\vtheta',\vtheta_t)^2/(2\tau)\) over \(\vtheta'\in\Gamma_p^d\) with time step \(\tau\) \citep{mayer1998gradient,ambrosio2008gradient,jordan1998variational}; our analytic extensions give directional derivatives directly, so we take explicit steps. The minimization in \cref{eq:gd} is not of this kind: it is over the tangent cone, where
\(D_{\vv}L\) is piecewise linear in the speeds (\Cref{sec:coordinate-implementation}). Subgradient methods on nonpositively curved spaces \citep{goodwin2026subgradient} assume geodesic convexity, which our regression loss lacks even for affine models (\Cref{rem:affine-nonconvexity}).

\label{sec:optimizers}

\textbf{\textit{Momentum}} in Euclidean space averages past gradients, \(m_t=\beta m_{t-1}+(1-\beta)\nabla L(\vtheta_t)\) and steps along \(-m_t\) in place of \(-\nabla L(\vtheta_t)\). On \(\Gamma_p\) a coordinate has one slope per available direction, the set of which changes at each vertex.
For a single coordinate write \(\ve_q\) for its unit velocity in direction \(q\).
We average \(D_{\ve_+}L(\vtheta_t)\) and \(D_{\ve_-}L(\vtheta_t)\) over every update of the coordinate; at a vertex, where direction \(-\) does not exist, we use the smallest child slope \(\min_\rho D_{\ve_\rho}L(\vtheta_t)\) in its place. The amount by which child \(\rho\) exceeds it, \(\Delta_t(\rho)=D_{\ve_\rho}L(\vtheta_t)-\min_{\rho'}D_{\ve_{\rho'}}L(\vtheta_t)\ge0\), is averaged only over the updates made at that vertex. Write \(u_t,d_t\) for the upward and downward averages and \(\bar\Delta_t(\rho)\) for the child's average excess. The momentum estimates are \(m_t(+)=u_t\), \(m_t(-)=d_t\), and \(m_t(\rho)=d_t+\bar\Delta_t(\rho)\).

The speed \(s_t(q)=\max(0,-m_t(q))\) takes the place of \(-D_{\vv^*}L\) in \cref{eq:gd}.
Set \(s_t(\rho)=0\) if the child's directional derivative, averaged over updates at its vertex, is nonnegative. On an edge where \(L\) is differentiable, \(D_{\ve_-}L=-D_{\ve_+}L\). If the incoming histories also satisfy \(d_{t-1}=-u_{t-1}\), the recurrences preserve \(d_t=-u_t\) and reduce to Euclidean momentum.

\textbf{\textit{Adam}} divides speed by the root of a second-moment average. We keep one scalar second moment per parameter, so the division rescales every direction of a coordinate alike and cannot change which one is chosen, as separate scales would; \citet{cho2017riemannian} assign one adaptive rate per weight vector for the same reason.

See \Cref{app:adaptive-updates} for all averaging, second-moment, and bias correction equations. Both extensions are equivariant in
distribution under factorwise tree isometries and coordinate permutations
(\Cref{prop:adaptive-equivariance}) and keeps unselected coordinates' data fixed.

\subsection{Approximation via Parameter Groups}
\label{sec:coordinate-implementation}

At an inner vertex a parameter has \(p+1\) branch choices, so naive joint selection over \(d\) parameters enumerates up to \((p+1)^d\) combinations. Two coordinates are \emph{coupled} by an output when its contribution to \(D_{\vv}L\) cannot be written as a function of one plus a function of the other. When several terms attain the maximum in \cref{eq:pushforward}, \((|\mathcal A_i|>1)\), the rate in \cref{eq:chain-rule} is a maximum over them and couples the input coordinates those terms involve, the \(j\) with \(I_j>0\) (\Cref{ex:tie-couples}).

Even a single active term can couple parameters. For example,
\(\theta_1\theta_2\) has only one term when expanded at centers
\((0,0)\). From the Gauss points
\(\theta_1=\theta_2=\zeta_{0,1}\), moving either parameter alone
down toward digit \(1\) leaves their product unchanged, whereas
moving both moves the product down toward digit \(1\).
At the child centers \((1,1)\), the expansion becomes
\((1+u)(1+v)=1+u+v+uv\). The output radius is then
\(\max(r_1,r_2)\): its derivative is zero when either parameter
moves alone, but negative when both move down together.
\textbf{Hence, the joint radius derivative cannot in general be written
as a sum of coordinate derivatives within a coupled group.}

Two parameters are in the same \emph{coupled parameter group} when a chain of couplings, over any outputs and in-batch examples, connects them. Each contribution then depends on one group only, so \(D_{\vv}L=\sum_{G\in\mathcal G}\psi_G(\vv_G)\), with \(\vv_G\) the velocity restricted to \(G\). Enumeration is linear in the number of groups \(\mathcal G\).
An output whose contribution is locally constant, such as a logit with \(|\hat c_k|_p>\hat r_k\) (\Cref{sec:classification}), contributes zero to \(D_{\vv}L\).
For affine models each term has one coordinate, so coupling comes only from ties, and the groups are read off the forward pass in linear time. The affine implementation can further split these groups when contributions cancel exactly (\Cref{app:directional-groups,app:implementation_details}).%

Enumeration can still be exponential per group,
so each update proceeds as follows:
\begin{enumerate}[nosep,leftmargin=*]
\item Select one coordinate per group uniformly among those awaiting a turn (\S\ref{app:optimizer-updates}).
\item On a common minibatch, compute each selected coordinate's directional slopes. For gradient descent, choose a minimizing direction with speed \(\max(0,-\mu_G)\), where \(\mu_G\) is its least unit-direction slope. This suffices because derivatives add across groups (\S\ref{app:coordinate-restriction}).
\item Move the selected coordinates together under the clipping rule. A stationary selection still counts as a turn.
\end{enumerate}
With fixed groups this is a random-permutation sweep~\citep{lee2019randompermutations}, as in parallel coordinate descent~\citep{liu2015asynchronous}; a single group gives ordinary coordinate descent.

\subsection{Backpropagation}
\label{sec:backward-propagation}

To choose a parameter direction we need its effect on the loss, which one can compute stagewise via backpropagation \citep{linnainmaa1976taylor}. Since directional derivatives add over examples, fix one and write \(L=\ell\circ F\), with \(F=F_K\circ\cdots\circ F_1\) as in \Cref{sec:forward_propagation} and \(\ell\) real-valued. Each stage maps the directions and speeds of its inputs to those of its results.
An input radius at stage \(k\) moves at \(D_{\vv}r_j=\pm s_j\) (\(+\) up, \(-\) down or down via a digit).
For each selected child direction, form the expansion and active terms in \cref{eq:pushforward} using the center \(c_j+\rho_jp^{n_j}\) of that child (\cref{eq:unit-paths}; any representative gives the same result, \Cref{rem:rep-invariance}). Differentiating with these representatives fixed gives the rates \(D_{\vv}\hat{\vr}\); working backwards from the loss head (\Cref{prop:directional-chain-rule}), for \(k=K,\dotsc,1\):
\begin{equation}
\label{eq:chain-rule}
D_{\vv}\bigl(\ell\circ\tilde F_K\circ\cdots\circ\tilde F_k\bigr)
= D_{\hat{\vv}}\bigl(\ell\circ\tilde F_K\circ\cdots\circ\tilde F_{k+1}\bigr),
\qquad
D_{\vv}\hat r_i=\hat r_i\max_{\vI\in\mathcal A_i}\sum_j \frac{I_j\,D_{\vv}r_j}{r_j},
\end{equation}
with \(\mathcal A_i\) the active terms and $\hat{r}_i$ the output radius of \cref{eq:pushforward}.
For positive input radii, the radius-rate formula only applies when \(\hat r_i>0\); when \(\hat r_i=0\), the radius derivative is zero.
The resulting velocity \(\hat{\vv}=(\hat{\vq},\hat{\vs})\) has speeds \(\hat s_i=|D_{\vv}\hat r_i|\). A positive rate moves up, and a negative rate moves down toward \(F_{k,i}(\vc')\), where \(\vc'\) is \(\vc\) with the selected child centers substituted (\Cref{app:phad}).
When each \(\mathcal A_i\) has one term, \cref{eq:chain-rule} is linear in the input rates, as in backpropagation through a linear layer; ties are treated in \Cref{sec:coordinate-implementation}.
The stage-by-stage extension is exact in every center and a bound in every radius (\Cref{sec:forward_propagation}); the result equals \(D_{\vv}(\ell\circ\tilde F)\) when no parameter reaches a stage along two paths. Weighted over examples, these are the derivatives \cref{eq:gd} requires for the loss evaluated on the stagewise output.

\paragraph{Affine models.} For affine models, the ordinary Jacobian%
approximation \(F(\vc)+J_F(\vc)(u-\vc)\) equals \(F\), so the same
disk extension and loss give the original directional derivatives for both
linear regression and classification. At a vertex, each output scalar has a
common child derivative and at most one exception, so accumulating these with
the signed loss-derivative weights assembles the \(p\) child slopes in
\(O(m+p)\) rather than \(O(mp)\) work for \(m\) contributing outputs.
We may view this as ``Jacobian descent'': select the gradient step using
the loss composed with the Jacobian approximation. See
Appendix~\ref{app:coincidence} for the formulas and
Appendix~\ref{app:evaluation-cost} for the cost analysis.

\section{\texorpdfstring{\(p\)}{p}-adic Regression}
\label{sec:regression}

We minimize mean \(p\)-adic error, or \(L_1\) loss,
following \citet{baker2022number,baker2025linear} and \citet{zubarev2025adic}; exact minimization is NP-hard for \(p=2\) when the number of coefficients is part of the input \citep{baker2026nphardness}.
The \(L_\infty\) loss instead minimizes worst-case error, as in \citet{martins2025learning}.
Squared \(L_2\) loss is also possible, but the usual least-squares justification does not carry over: the \(p\)-adic norm is not induced by an inner product, so its square gives neither a quadratic objective nor a closed-form solution and remains locally constant at nonzero residuals.
For training, we use the \textit{mean direct loss} $L_d$, the continuous relaxation from \cref{eq:direct_loss_def}:
\begin{equation}
L_1(\vtheta) = \frac1\nex\sum_i|f(\vx^{(i)};\vtheta)-y^{(i)}|_p \quad\longrightarrow\quad L_d(\vtheta)=\frac1\nex\sum_i\ell_d(f(\vx^{(i)};\vtheta),y^{(i)}).
\end{equation}

\begin{table}[t]
\centering
\small
\caption{Affine and two-layer regression under 3-adic \(L_1\) loss (mean absolute error), with training-example equivalents needed for recovery at depths 4 and 5. Entries are mean \(\pm\) SD over five coefficient instances; NR means $x$/5 recovered. N/A: beam search always fixes digits from the coarsest downward. \textbf{Bold} results overlap SDs with the best.
}
\label{tab:regression}
\resizebox{\linewidth}{!}{%
\setlength{\tabcolsep}{2pt}
\begin{tabular}{@{}lccc@{\hspace{8pt}}ccc@{\hspace{8pt}}ccc@{}}
\toprule
\multirow{3}{*}{\textbf{Optimizer}}
 & \multicolumn{3}{c}{\textbf{Centers initialized at zero}}
 & \multicolumn{3}{c}{\textbf{Adverse start}}
 & \multicolumn{3}{c}{\textbf{Two-layer target}} \\
\cmidrule(lr){2-4}\cmidrule(lr){5-7}\cmidrule(l){8-10}
 & \multirow{2}{*}{\shortstack{Test \(L_1\) loss\\(\(\log_3\))}} & \multicolumn{2}{c}{Samples\(/N\)}
 & \multirow{2}{*}{\shortstack{Test \(L_1\) loss\\(\(\log_3\))}} & \multicolumn{2}{c}{Samples\(/N\)}
 & \multirow{2}{*}{\shortstack{Test \(L_1\) loss\\(\(\log_3\))}} & \multicolumn{2}{c}{Samples\(/N\)} \\
 & & depth 4 & depth 5 & & depth 4 & depth 5 & & depth 4 & depth 5 \\
\midrule
Beam (selected width) & \(\mathbf{-5.26 \pm 0.01}\) & \(27 \pm 0\) & \(36 \pm 0\) & \multicolumn{3}{c}{N/A} & \(-5.25 \pm 0.01\) & \(96 \pm 34\) & \(127 \pm 45\) \\
GD (approx.) & \(\mathbf{-5.27 \pm 0.01}\) & \(\mathbf{0.8 \pm 0.1}\) & \(\mathbf{1.0 \pm 0.1}\) & \(-3.27 \pm 2.74\) & \(\mathrm{NR}\;(3/5)\) & \(\mathrm{NR}\;(3/5)\) & \(\mathbf{-5.26 \pm 0.02}\) & \(\mathbf{1.1 \pm 0.2}\) & \(\mathbf{1.5 \pm 0.2}\) \\
+ Momentum & \(\mathbf{-5.27 \pm 0.01}\) & \(\mathbf{0.8 \pm 0.0}\) & \(\mathbf{1.0 \pm 0.1}\) & \(\mathbf{-5.27 \pm 0.01}\) & \(\mathbf{2.9 \pm 0.3}\) & \(\mathbf{3.1 \pm 0.3}\) & \(\mathbf{-5.27 \pm 0.01}\) & \(\mathbf{3.7 \pm 4.4}\) & \(\mathbf{6.7 \pm 6.9}\) \\
+ Adam & \(\mathbf{-5.27 \pm 0.01}\) & \(\mathbf{0.8 \pm 0.1}\) & \(\mathbf{1.0 \pm 0.1}\) & \(\mathbf{-5.27 \pm 0.01}\) & \(\mathbf{2.7 \pm 0.1}\) & \(\mathbf{2.9 \pm 0.2}\) & \(\mathbf{-5.26 \pm 0.02}\) & \(\mathbf{3.0 \pm 2.5}\) & \(\mathbf{4.8 \pm 4.7}\) \\
\bottomrule
\end{tabular}%
}
\end{table}

\paragraph{Experiments.}
We evaluate on a 3-dimensional linear regression task in \(\Q_3\). For each seed, we draw \(\vtheta^*\) and each input vector \(\vx^{(i)}\) uniformly
from \(\{0,\ldots,3^8-1\}^3\) and generate targets as
\(y^{(i)}={\vtheta^{*\top}}\vx^{(i)}+3^5\eta^{(i)}\). Let \(\eta^{(i)}\) be a uniform integer in \(\{1, \dots, 99\}\), representing noise at depth five. We also train a two-layer model on four-input targets generated from
\(f(\vx)=v_1(x_1+w_1x_2)^2+v_2(x_3+w_2x_4)^2\)
with the same noise.
Both gradient training and beam search use this four-coefficient model;
all four gradient parameters start at the Gauss point (\Cref{app:two-layer-target}). 512 train/dev/test examples are drawn, and training uses batches of 32. At the true coefficients the expected \(L_1\) loss is
\(-5.26\) on the \(\log_3\) scale. We report \(L_1\) error using the
representatives retained by training, with their radii set to zero
(\(L_1\) is not representative-invariant).

We compare approximate gradient descent (GD), momentum, and Adam with \citet{martins2025learning}'s classifier beam search (their Algorithm~1), adapted to regression with widths five and ten and batch-32 scoring (\Cref{app:training-protocol,app:regression-beam}); their own regression method, which is for $L_\infty$, is not run. %
For each coefficient instance, learning rates and beam widths are selected by work to first depth-five agreement on every independent validation observation of the same function (\Cref{app:training-protocol}). Training work counts repeated uses of examples: each gradient update or
beam-candidate scoring on a batch of \(b\) examples contributes \(b/512\),
regardless of how many coordinates move. Arithmetic costs are given in
\Cref{app:evaluation-cost}.

For affine regression, our gradient methods place every
coefficient at the \textit{Gauss point} \(\zeta_{0,1}\), the disk of radius one
around zero and the uninformed state of the tree; beam search starts with every
coefficient at zero. We also test an \emph{adverse} start: a digitwise local
minimum, at which no change of any single digit lowers the training loss, lifted
to disks of radius \(3^{-8}\) so that radius growth can revisit coarser digits
(\Cref{app:regression-protocol}).

\paragraph{Results.}

See \Cref{tab:regression}. Recovery through depth \(q\) means
\(|c_j-\theta_j^*|_p\le p^{-q}\) for every coefficient \(j\). With centers initialized at zero, beam search and the gradient
descent methods all reach test losses close to the loss at the true
coefficients. %
With batches of 32, the gradient methods reach depth 5 after processing 0.95--0.98 training-example equivalents on average, compared with 35.5 for beam search of width five. With the engineered adversarial start, gradient methods often (but not always) escaped the local minima. Here, plain GD (approx.) recovered on three of five seeds, while Momentum and Adam recovered on all five.

For the two-layer network, recovery is over the coefficients of \(f\) as a
quadratic form. In contrast, a bias-free linear control trained by GD on the same data at \(\kappa=1\)
has test loss \(-0.49\pm0.16\), against \(-0.36\pm0.19\) for predicting zero. %
GD recovered on all five seeds in \(1.45\pm0.25\)
training-example equivalents, with test loss close to the loss at the true
coefficients; Momentum and Adam also recovered all five. Beam search recovers
all five with the selected widths; width five alone misses one
(\Cref{app:two-layer-target}).
On the missed seed the target \(v_1\) is divisible by nine, so a candidate
matching it has \(v_1=0\) after two digits; its term \(v_1(x_1+w_1x_2)^2\) is
then zero whatever its \(w_1\), and nine candidates share one loss exactly.
At the next digit these nine separate and take the nine lowest losses, with
the one matching the target last of them, so width five keeps only the
first five.

\section{\(p\)-adic Classification}
\label{sec:classification}

In Euclidean machine learning, a linear binary classifier assigns a
label by testing whether \(f(\vx;\vw)=\vw^\top\vx+b\ge0\).
Over \(\Qp\), the output $p$-adic values have no coherent ordering \Cref{sec:background}. Hence, \citet{martins2025learning} defined
\textit{\(p\)-adic linear classifiers} as
\(|\vw^\top\vx+b|_p\le1\)---equivalently, \(v_p(\vw^\top\vx+b)\ge0\).
\citeauthor{martins2025learning} trained with beam search; a similar 2-adic threshold unit, trained by random search, appears in \citet{khrennikov2000learning}.
We extend this to multiclass classification by defining softmax which we optimize with
gradient descent in $\Gamma_p^d$, resolving their Open Problem 6.1 (linear multiclass \(p\)-adic classifiers) and addressing 6.2 (gradient-based optimization).

\subsection{Logits, Softmax, and Loss}

For \(C\) classes, stack affine maps into
\(f(\vx)=\mW\vx+\vb\in\Qp^C\), where row \(k\) of \(\mW\) is
\(\vw_k^\top\) and \(f_k(\vx; \vw_k, b_k)=\vw_k^\top\vx+b_k\).
Generalizing Martins' binary classifier, our goal at inference is to have the real-valued \(z_k=v_p(f_k(\vx))=-\log_p|f_k(\vx)|_p\) act as logit
for a softmax on \(\vz=(z_1,\ldots,z_C)\).

During training, the map for class $k$ extends to $\tilde{f}_k(\vx; \vw_k, b_k)$ whose aggregate parameter state is in $\Gamma_p^{\nparams/C}$.
Writing its output as \(\zeta_{\hat c_k,\hat r_k}\),
we take its \textit{logit} to be the negative base-\(p\) logarithm of the radius of
its LCA with zero, an upper bound on \(|u|_p\) over the representatives \(u\) of the output (\cref{eq:relation}):
\(z_k=-\log_p\max\{|\hat c_k|_p,\hat r_k\}\). When \(|\hat c_k|_p>\hat r_k\) this is \(v_p(\hat c_k)\), unchanged by the radius; else it is \(-\log_p\hat r_k\).
The direct loss (\cref{eq:direct_loss_def}) is a generalization of distance; in contrast, \(|f_k(\vx)|_p\le1\) is a condition on \textit{all} representatives, which \(z_k\ge0\) guarantees all satisfy.

Converting the valuation logits to natural-log
units, applying softmax, and simplifying gives the class probabilities and cross-entropy
loss per example $i$ and in aggregate $L_\text{CE}$:
\begin{equation}
\pi_k^{(i)}=
\frac{\bigl(\max\{|\hat c_k^{(i)}|_p,\hat r_k^{(i)}\}\bigr)^{-1/T}}
{\sum_{k'}\bigl(\max\{|\hat c_{k'}^{(i)}|_p,\hat r_{k'}^{(i)}\}\bigr)^{-1/T}},
\qquad
\ell^{(i)}_{\mathrm{CE}}=-\ln\pi^{(i)}_{y^{(i)}},
\qquad
 L_\text{CE} = \frac{1}{\nex}\sum_i \ell^{(i)}_\text{CE}
\end{equation}
where $T$ is softmax temperature, giving probabilities \(\operatorname{softmax}((\ln p)z^{(i)}/T)\).
We use \(T=1\) throughout. The binary
formulation appears in \Cref{sec:quillian}.
For \Cref{tab:modulo,tab:quillian}, we extract \(p\)-adic parameters from the learned parameter states using the convention in \Cref{app:classification-point-evaluation}.
For a finite \(p\)-adically separable training set and a classifier with independent output biases, \Cref{thm:margins} constructs states with fixed separating centers, positive parameter radii, and any prescribed common output radius \(\rho\in(0,1]\). Every correct-class logit then exceeds every other by at least \(1\).

\subsection{Experiments}

\paragraph{Modulo Classification.}

For \(x \bmod p^k\), inputs with the same remainder share a base-\(p\) prefix. %
A closed-form affine solution follows by applying \citet[\S 3.1, Proposition~2]{martins2025learning} to each residue class, but the task was not trained there. Our initialization selects the weight's depth from the training labels but is identical for every class, so it predicts at chance; we test whether training finds the class biases that complete such a solution, under our cross-entropy objective with valuation logits; further setup in \Cref{app:training-protocol,app:modulo-protocol}.

\begin{table}[t]
\centering
\small
\caption{Modulo classification test accuracy (mean \(\pm\) SD over 5 seeds). All models use Adam (Euclidean, or ours for $p$-adic) with batch size 32.
}
\label{tab:modulo}
\footnotesize
\begin{tabular}{@{}lcc@{}}
\toprule
\textbf{Representation / model prime}
  & \textbf{Task: \(x\bmod 4\) (acc.)} & \textbf{Task: \(x\bmod 9\) (acc.)} \\
\midrule
\(2\)-adic (matched for mod 4) & \(\mathbf{1.000 \pm 0.000}\) & \(0.111 \pm 0.000\) \\
\(3\)-adic (matched for mod 9) & \(0.250 \pm 0.000\) & \(\mathbf{1.000 \pm 0.000}\) \\
\(5\)-adic (mismatched control) & \(0.250 \pm 0.000\) & \(0.112 \pm 0.005\) \\
Matched with permuted codes & \(0.250 \pm 0.000\) & \(0.112 \pm 0.002\) \\
Euclidean raw scalar & \(0.245 \pm 0.010\) & \(0.110 \pm 0.002\) \\
\midrule
\emph{Euclidean one-hot digits} & \(1.000 \pm 0.000\) & \(1.000 \pm 0.000\) \\
\emph{Chance reference} & \(1/4\) & \(1/9\) \\
\bottomrule
\end{tabular}%
\end{table}

In Table~\ref{tab:modulo}, matched-prime classifiers attain perfect test accuracy across five seeds, while mismatched primes, permuted codes, and Euclidean linear classifiers on raw scalar inputs remain near chance. These results show the importance of an encoding that linearly separates congruence classes, which for $p$-adics is prime-dependent; one-hot encodings do this (at the cost of dimension) and also fit.

\paragraph{Quillian Semantic Network Benchmark.}
\label{sec:quillian}

We evaluate on the 1{,}680-proposition Quillian Semantic Network
\citep{Quillian1968} as adapted by \citet{martins2025learning}. We use
Martins's hand-designed \(\mathbb{Q}_2\) entity codes, fixed before
optimization,
in which semantically close entities share long base-2 prefixes. Each
of the 28 properties is an independent binary task with a sigmoid
head; only 8.04\% of labels are positive, so we report positive-class
F1 and average precision (AP), pooled across properties within each
test split, and accuracy for comparison with Martins. 

Each head uses an affine map of the three entity-code coordinates
and four relation indicators; the active relation's weight supplies
the bias. Writing \(S=\max\{\hat r,|\hat c|_p\}\), the two valuation
logits \(z_0=0\) and \(z_1=-\log_p S\) give \(\pi_1=1/(1+S^{1/T})\) and
binary cross-entropy at \(T=1\).
An attribute is predicted present when \(\pi_1\ge1/2\), equivalently
\(S\le1\), recovering Martins's unit-threshold rule at zero radius.
Beam search ranks and prunes candidates using training labels; we also use them to select an initial input coordinate and digit depth \(v\). The selected coefficient starts at center \(p^{-v}\) and the others at zero, fixing an initial representative among classifiers equivalent under multiplication by a scalar of \(p\)-adic norm one. All coefficients remain trainable (\Cref{app:training-protocol}).

\begin{table}[t]
\centering
\small
\caption{Quillian benchmark test results, pooled across 28 binary
properties within each split (mean \(\pm\) SD over five random
80/20 splits). (full) means no batching. \textbf{Bold} marks the highest mean per metric within each model family. \textit{Italics} are our higher-accuracy reproductions of
\citet{martins2025learning}; see \Cref{app:quillian-beam}.
}
\label{tab:quillian}
\footnotesize
\begin{tabular}{@{}lccc@{}}
\toprule
\textbf{Representation, optimizer} & \textbf{Pooled F1}
  & \textbf{Pooled AP} & \textbf{Accuracy} \\
\midrule
\textit{\(2\)-adic linear, beam search (W=10, depth 8, full)}
  & \(\mathbf{0.875 \pm 0.084}\) & \(\mathbf{0.826 \pm 0.111}\) & \(\mathbf{0.981 \pm 0.012}\) \\
\(2\)-adic linear, beam search
  & \(0.845 \pm 0.082\) & \(0.804 \pm 0.106\) & \(0.976 \pm 0.011\) \\
\(2\)-adic linear, Adam (ours)
  & \(0.804 \pm 0.034\) & \(0.816 \pm 0.037\) & \(0.973 \pm 0.004\) \\
Permuted \(2\)-adic linear, Adam (ours)
  & \(0.470 \pm 0.092\) & \(0.551 \pm 0.077\) & \(0.915 \pm 0.014\) \\
\midrule
\textit{Euclidean MLP, Adam (full)}
  & \(\mathbf{0.844 \pm 0.053}\) & \(\mathbf{0.879 \pm 0.042}\) & \(\mathbf{0.976 \pm 0.008}\) \\
\midrule
Always negative, all scores equal
  & \(0.000 \pm 0.000\) & \(0.080 \pm 0.004\) & \(0.920 \pm 0.004\) \\
\bottomrule
\end{tabular}%
\end{table}

\Cref{tab:quillian} compares our training procedure with beam search on Martins's architecture and with their roughly three-times larger, two-layer Euclidean MLP (\Cref{app:quillian-protocol}). \citet{martins2025learning} also give hand-set weights that classify every proposition correctly, so test error here reflects learning from the training split, for every method. Gradient descent approaches both beam baselines in accuracy. As with modulo, the permuted control shows that giving Quillian data a \(2\)-adic encoding that preserves its hierarchy provides a useful inductive bias for gradient descent.

\section{Conclusion}

 The framework presented here trains affine models and one restricted two-layer model with \(p\)-adic parameters by
approximate joint gradient descent, using exact derivatives. For affine models
these derivatives can be accumulated sparsely (\Cref{app:coincidence}).
Though we describe the theory for arbitrary-depth models, the tradeoff between the radius bound versus the cost of collecting stages, and the ideal architecture and setup for deep models remain open questions.
As suggested by \citet[\S 6.4]{martins2025learning}, we intend to develop adelic formulations of losses and descent, capturing the advantages of both real and $p$-adic norms.
Over Berkovich spaces, this would be Arakelov geometry \citep{chambert2006mesures}.
Convergence guarantees remain open. See \Cref{app:disclosures} for our work's limitations and use of AI.

\paragraph{Acknowledgements.} We thank Micha\"{e}l Sander, Jessica Hoffmann, Maxime Guigon, and the anonymous reviewers and area chair for their helpful feedback. We thank the AI co-mathematician team at Google DeepMind for providing model access and support. The first author thanks Clifton Cunningham for a guided reading of \textit{A Course on Arithmetic} in 2014-15, and congratulates Jean-Pierre Serre on his 100th birthday.

{
\small

\bibliographystyle{plainnat}
\bibliography{references}

}

\appendix
\crefalias{section}{appendix}
\crefalias{subsection}{subappendix}
\Crefname{appendix}{Appendix}{Appendices}
\Crefname{subappendix}{Appendix}{Appendices}

\section{Disclosures}
\label{app:disclosures}

\paragraph{Limitations.} We do not yet present formal guarantees of convergence, and our nonlinear experiments cover only one two-layer architecture. Restricting participation to one coordinate per group can miss descent
directions that require several coordinates in that group to move together. Exact center arithmetic (\Cref{app:arithmetic-resolution}) restricts
\(\Qp\) values to \(\mathbb{Z}[1/p]\) up to a \(p\)-dependent precision. Our evaluation is small in scale. Larger \(p\)-adic experiments exist for exactly solved one-coefficient regression \citep{baker2022number} and for fitting fixed taxonomies \citep{nguessan2025vpunns}; evaluating our optimizers at that scale is future work.

\paragraph{AI use.} We used AI models, including Gemini 3.x \citep{comanici2025gemini} and the AI co-mathematician \citep{zheng2026ai}, in preparing this publication. We drove the research direction, conceptual foundations, exposition, and experiment design. AI drove the implementation, technical proofs, and formal presentation, which we revised. We accept responsibility for the contents of this work.

\section{Comparison with \(p\)-adic Optimization Methods}
\label{app:related}
\label{app:cost-accounting}

\suppressfloats[t]
\begin{table}[t]
\centering
\small
\caption{Cost of choosing one update for \(p\)-adic optimizers, in evaluations of the loss or of one directional derivative on the examples used for that update. More details per method in \Cref{app:related}. \(\nparams\): number of model parameters; \(E\): \(p\)-adic digits searched per parameter; \(W\): beam width; \(\beta>0\): the random walk's weighting parameter. \(^{\dagger}\)Draws are uniform, so total work matches exhaustive search.}
\label{tab:complexity}
\begin{tabularx}{\linewidth}{@{}>{\raggedright\arraybackslash}X >{\raggedright\arraybackslash}X r@{}}
\toprule
\textbf{Method} & \textbf{Update choice} & \textbf{Cost} \\
\midrule
\multicolumn{3}{@{}l}{\emph{Zeroth-order, discrete: moves chosen by model evaluation at candidates}}\\
\hspace{1em}Exhaustive search & every digit assignment & \(p^{E \nparams}\) \\
\hspace{1em}Beam search \citep{martins2025learning} & next digit of all parameters; keep best \(W\) & \(O(W p^{\nparams})\) \\
\hspace{1em}Greedy search \citep{mihara2026padicpolynomial} & best value per coefficient & \(O(p^E)\) \\
\hspace{1em}Random digit search \citep{khrennikov2000learning} & uniformly random digit vectors & \(O(1)\) \\
\hspace{1em}Weighted random walk \citep{zubarev2025adic} & loss-weighted random increments & \(\le e^{\beta}\) expected\(^{\dagger}\) \\
\hspace{1em}Simulated annealing \citep{mihara2026padicpolynomial} & annealed coefficient replacement & \(O(1)\) \\
\midrule
\multicolumn{3}{@{}l}{\emph{First-order, continuous: moves chosen by local tangent information}}\\
\hspace{1em}Joint GD (ours) & naive enumeration of branch tuples & \((p+1)^{\nparams}\) tuples at vertices \\
\hspace{1em}GD (approx.) (ours) & steepest direction of one coordinate per group & \(O(p\nparams)\) \\
\hspace{1em}Coord.~GD (ours) & steepest direction of one coordinate & \(O(p)\) \\
\bottomrule
\end{tabularx}
\end{table}

\paragraph{How costs are counted.}
A directional derivative costs one pass through the model, as a loss evaluation
does (\Cref{app:phad}); \Cref{app:evaluation-cost} gives the cost of a pass for
our regression experiments, where the derivatives share the forward computation,
so counting a separate pass for each gives upper bounds for coordinate and
grouped GD. Entries are per update and do not say how many updates an optimizer
needs.

The methods below, except those in the last paragraph, optimize
classical \(p\)-adic values or finite digit prefixes. Gradient descent
permits positive-radius states in the Berkovich hull.

\paragraph{Classifier beam search.} \citet{martins2025learning}'s Algorithm 1 trains a binary
linear classifier by searching the base-\(p\) digits of its weights. A node at depth \(\delta\) represents the first
\(\delta\) digits of all \(\nparams\) coefficients. At the next depth,
every retained node is expanded into its \(p^{\nparams}\) possible
next-digit vectors, inadmissible candidates are pruned using the
positive-error constraint, and the \(W\) candidates with the fewest
total mistakes are kept.
Martins reports runtime \(O(WEnp^{\nparams})\) when each
candidate is evaluated on \(n\) examples. Counting each candidate's loss
value as one evaluation, one depth costs \(O(Wp^{\nparams})\), the entry
in \Cref{tab:complexity}. A finite beam can discard the
prefix of the best later solution, so the procedure is an approximate,
depth-limited prefix search rather than a local optimizer in the usual
gradient sense.

\paragraph{Hamming-graph heuristics.} \citet{mihara2026padicpolynomial} represents the polynomial
coefficients in \((\mathbb Z/p^E\mathbb Z)^{\nparams}\) and optimizes a
lexicographic vector that first counts errors in the earliest digit,
then the next, and so on.
A Hamming graph lets an update replace one
coefficient residue while holding the others fixed: a greedy sweep tries
all \(p^E\) values for each coefficient, for \(\nparams p^E\) loss
updates plus data and modular-arithmetic cost; hill climbing repeatedly
proposes a coefficient and a replacement value; simulated annealing uses
the same proposals but may accept a worsening move under its temperature
schedule.

\paragraph{Random digit search.} \citet{khrennikov2000learning} train one 2-adic threshold unit, which outputs \(0\) when \(|\vw^\top\vx|_2\le2^{-k}\) and \(1\) otherwise, and minimize the fraction of wrong outputs. Weights are restricted to \(k\) binary digits, a grid of \(2^{k\nparams}\) vectors. For each example they sample the grid uniformly until a vector gives the right output, then sample spheres of doubling 2-adic radius around it to collect more such vectors. They then sample from these sets and keep the vector that is right on the most sampled examples. Each draw costs one evaluation on one example, the \(O(1)\) entry of \Cref{tab:complexity}; when the grid is small they enumerate it instead.

\paragraph{Weighted random walk.} \citet{zubarev2025adic} studies a
Mahler-polynomial model with the arithmetic mean of \(p\)-adic residual
distances. The update adds a whole-vector increment
\(\xi\) whose conditional density relative to Haar measure on
\(\mathbb Z_p^{\nparams}\) is proportional to
\(\exp\{-\beta[L(w+\xi)-L(w)]\}\).
Under the assumption that improving
increments have positive Haar measure, an appropriate range of \(\beta\)
gives strict expected one-step descent. \citet{zubarev2025adic} gives an update rule but no sampler. Since the loss is nonnegative, rejection sampling from Haar measure with acceptance probability \(\exp(-\beta L(w+\xi))\), in their notation, draws it exactly at one loss evaluation per proposal. Since the loss is at most one for data and parameters in \(\Zp\), the expected number of proposals per step is at most \(e^{\beta}\). For fixed \(\beta\), the random-walk row of \Cref{tab:complexity} therefore lists an expected cost independent of \(\nparams\), \(E\), and \(p\); their expected-descent result holds only for a range of \(\beta\) that depends on the current point. Each proposal is a uniform draw from \(\Zp^{\nparams}\), so fixing \(E\) digits of every coefficient takes as many proposals in expectation as exhaustive search.

\paragraph{Methods that optimize real numbers.} Real-parameter models appear in our experiments as the Euclidean rows of \Cref{tab:modulo,tab:quillian}; we do not claim that \(p\)-adic models outperform them, and the one-hot model also fits the modulo tasks. We do not run the following published methods. \citet{zuniga2024deep} trains real weights by ordinary gradient descent; the group's implemented example trains real-valued weights on unlabeled MNIST by contrastive divergence and shows reconstructions \citep{zuniga2023boltzmann}. Adam--VAPO \citep{nguessan2025vpunns} runs Adam on one real latent per weight and rounds it to a base-\(p\) digit; its gradient is given for v-PuNN's per-digit heads, which map a parent digit to a child digit in a fixed hierarchy, and its greedy variant compares a digit with its two neighbors modulo \(p\). Hyperbolic embeddings \citep{nickelKiela2017} and ultrametric fitting \citep{chierchia2019ultrametric} learn real coordinates or real edge weights for given data points and have no model \(f(\vx;\vtheta)\) to train on our tasks.

\section{Formal Geometric Definitions}
\label{app:formal-geometry}

A Euclidean parameter inherits its possible directions and unit speed
from its ambient vector space. Gradient descent needs the corresponding data on the
Berkovich hull.

\paragraph{The Berkovich affine line.} A multiplicative seminorm on a
ring \(A\) is a map \(|\cdot|_x:A\to\mathbb R_{\ge0}\) satisfying
\(|0|_x=0\), \(|1|_x=1\), \(|f+g|_x\le\max\{|f|_x,|g|_x\}\), and
\(|fg|_x=|f|_x|g|_x\). The Berkovich affine line over \(K=\mathbb Q_p\)
is
\[
\mathbb A^{1,\mathrm{an}}_K=\mathcal M(K[T]),
\]
the set of all multiplicative seminorms on \(K[T]\) whose restriction to
constants is the given absolute value, \(|c|_x=|c|_p\) for \(c\in K\).
It carries the coarsest topology for which \(x\mapsto|P|_x\) is
continuous for every \(P\in K[T]\). An ordinary point \(a\in K\) embeds
by evaluation, \(|P|_a:=|P(a)|_p\).

\begin{definition}[Disk points and representatives]
\label{def:disk-as-seminorm}
For \(a\in K\) and \(r>0\), the disk point \(\zeta_{a,r}\) is the
multiplicative seminorm
\[
\Bigl|\sum_{n=0}^{N}c_n(T-a)^n\Bigr|_{\zeta_{a,r}}
:=\max_n |c_n|_p\,r^n.
\]
We also write \(\zeta_{a,0}=a\) for evaluation at the ordinary point
\(a\).
\end{definition}

\begin{lemma}[Representatives]
\label{lem:representatives}
For \(r,s\ge0\), \(\zeta_{a,r}=\zeta_{b,s}\) if and only if \(r=s\) and \(|a-b|_p\le r\).
\end{lemma}

\begin{proof}
If the two seminorms are equal, evaluating \(T-a\) and \(T-b\) gives
\(r=\max\{s,|a-b|_p\}\) and \(s=\max\{r,|a-b|_p\}\); hence \(r=s\) and
\(|a-b|_p\le r\). Conversely, if \(|a-b|_p\le r\), expanding a
polynomial about either center and using the ultrametric inequality
gives the same Gauss seminorm.
\end{proof}

\begin{proposition}[The scalar hull and its arcs]
\label{lem:scalar-hull}
The hull \(\Gamma_p\) consists exactly of the points
\(\zeta_{a,r}\) with \(a\in\Qp\) and \(r\ge0\).
For \(x=\zeta_{a,r}\) and \(y=\zeta_{b,s}\), set
\(R=\max\{r,s,|a-b|_p\}\). Their arc in the Berkovich affine line is
\[
[x,y]=\{\zeta_{a,t}:r\le t\le R\}
       \cup\{\zeta_{b,t}:s\le t\le R\}.
\]
The two radius segments meet only at
\(\zeta_{a,R}=\zeta_{b,R}\).
\end{proposition}

\begin{proof}
Fix \(a\in\Qp\). For each polynomial
\(P(T)=\sum_{n=0}^N c_n(T-a)^n\), the function
\[
t\longmapsto |P|_{\zeta_{a,t}}
   =\max_{0\le n\le N}|c_n|_p t^n
\]
is continuous on \([0,\infty)\), with its value at zero interpreted
as \(|P(a)|_p\). By the defining topology of the Berkovich line,
\(t\mapsto\zeta_{a,t}\) is continuous. It is injective because
evaluating \(T-a\) recovers \(t\).

For \(u\in[r,R]\) and \(v\in[s,R]\), \Cref{lem:representatives}
gives \(\zeta_{a,u}=\zeta_{b,v}\) exactly when
\(u=v\ge|a-b|_p\). These bounds force \(u=v=R\).
Thus, unless \(x=y\), the two segments concatenate to a continuous
injective path from \(x\) to \(y\). The Berkovich line is Hausdorff
since polynomial evaluations separate its points, so this compact
path is homeomorphic to a closed interval. By
\citet[Proposition~I.6.12]{poineau2020berkovich}, the Berkovich
line over \(\Qp\) has a unique arc between any two distinct points.
Our constructed path is therefore \([x,y]\). When \(x=y\), both
segments are the same singleton.

Taking \(r=s=0\) shows that every arc between ordinary \(p\)-adic
points consists of points of the stated form. Conversely, given
\(\zeta_{a,r}\), choose an integer \(n\) with \(p^n\ge r\) and put
\(b=a+p^{-n}\). Since \(|a-b|_p=p^n\), the arc formula places
\(\zeta_{a,r}\) on \([a,b]\). Taking the union of these arcs proves
the asserted description of \(\Gamma_p\).
\end{proof}

For \(p=3\), \(\zeta_{0,1}\) is the Gauss point, and
\(\zeta_{0,1}=\zeta_{1,1}\) since \(|0-1|_3\le 1\).

The points of positive radius identify with the Bruhat--Tits tree of \(\mathrm{PGL}_2(\mathbb Q_p)\), the \(p\)-adic counterpart of the hyperbolic plane \citep{remy2014bruhat}. However, this tree has its own metric, under which adding \(\Qp\) as leaves makes its numbers infinitely distant.

\begin{definition}[Product states as seminorms]\label{rem:product-points}
A product state \((\zeta_{c_1,r_1}, \ldots, \zeta_{c_d,r_d})\) defines a multiplicative seminorm on \(\Qp[T_1, \ldots, T_d]\): writing \(f = \sum_I a_I (T - c)^I\), set \(|f| = \max_I |a_I|_p\, r_1^{i_1} \cdots r_d^{i_d}\)---the several-variable Gauss seminorm, multiplicative by the Gauss lemma \citep{bosch1984nonarchimedean}. Product states are therefore points of the \(d\)-dimensional analytification, and the forward rule of \Cref{sec:forward_propagation} evaluates exactly this seminorm. \end{definition}

\begin{example}[Coordinate images do not determine the joint state]
\label{ex:coordinate-marginals}
Scalar coordinate states do not, in general, determine a point of the higher-dimensional analytification. For example, let \(\zeta=\zeta_{0,1}\), let \(\xi\) be the product state \((\zeta,\zeta)\), and define \(|f|_\eta:=|f(T,T)|_\zeta\) for \(f\in\Qp[T_1,T_2]\). Both points restrict to \(\zeta\) on each coordinate, but
\[
|T_1-T_2|_\xi=1,\qquad |T_1-T_2|_\eta=0.
\]
At \(\eta\), the parameters satisfy \(T_1=T_2\) although both scalar radii equal one. The product state \(\xi\) has the same scalar restrictions but does not encode this relation.
\end{example}

The metric \(d_{\text{arc}}\) of \cref{eq:arc-metric} does not depend on the chosen center representatives (\Cref{lem:representatives}).

A metric space \((Y,d_Y)\) is \emph{hyperconvex} if every family of
closed balls \(\{\overline B_Y(y_i,r_i)\}_{i\in I}\) with \(r_i\ge0\)
satisfies
\[
  d_Y(y_i,y_j)\le r_i+r_j\ \text{ for all }i,j\in I
  \quad\Longrightarrow\quad
  \bigcap_{i\in I}\overline B_Y(y_i,r_i)\ne\varnothing.
\]
The \emph{tight span}, or \emph{metric envelope}, \(E(X)\) of a metric
space \(X\) is a hyperconvex space containing an isometric copy of \(X\),
with no proper hyperconvex subspace containing that copy. It is also
called the ordinary metric injective hull \citep[Secs.~2--3]{lang2013}.

\begin{theorem}[The \(p\)-adic metric envelope]
\label{thm:padic-injective-hull}
For the scalar hull \(\Gamma_p\) of \Cref{sec:parameter-space}, equipped
with the metric \(d_{\text{arc}}\) of \Cref{eq:arc-metric}, the map
\(x\mapsto\zeta_{x,0}\) realizes
\((\Gamma_p,\tfrac12d_{\text{arc}})\) as the tight span, or metric envelope,
of \((\Qp,d_p)\). Thus
\[
  E(\Qp,d_p)\cong(\Gamma_p,\tfrac12d_{\text{arc}}),
\]
with a unique isometry identifying each ordinary \(p\)-adic point with
its radius-zero disk point.
\end{theorem}

\begin{proof}
Use the disk points of \Cref{def:disk-as-seminorm}, with representatives
identified by \Cref{lem:representatives}. The arc description in
\Cref{eq:arc} and the length formula in \Cref{eq:arc-metric} make
\((\Gamma_p,\tfrac12d_{\text{arc}})\) an \(\R\)-tree. Indeed, finitely many
upward radius segments can be extended to a common height; they merge
and never separate, forming a finite tree whose edge lengths are half
the radius increments. Its distances agree with
\(\tfrac12d_{\text{arc}}\), so all metric triangles are tripods.
\Cref{eq:arc-metric} also gives
\(\tfrac12d_{\text{arc}}(\zeta_{x,0},\zeta_{y,0})=d_p(x,y)\).

It remains to prove completeness. Let \(C_M\) be the closed ball of
radius \(M\ge0\) about \(\zeta_{0,0}\) for
\(\tfrac12d_{\text{arc}}\). By \Cref{eq:arc-metric},
\[
  \zeta_{x,r}\in C_M
  \quad\Longrightarrow\quad
  \max\{|x|_p,r\}-\frac r2\le M
  \quad\Longrightarrow\quad
  r\le2M,\quad |x|_p\le2M.
\]
Consequently \(C_M\) is contained in the image of the compact set
\[
  \{x\in\Qp:|x|_p\le2M\}\times[0,2M]
\]
under \(q(x,r)=\zeta_{x,r}\). The same metric formula gives
\[
  \frac12d_{\text{arc}}(q(x,r),q(y,s))
  \le d_p(x,y)+\frac12|r-s|,
\]
so \(q\) is continuous. Hence \(C_M\), being a closed subset of that
compact image, is compact. Every Cauchy sequence lies in some \(C_M\),
and therefore converges in \(\Gamma_p\).

For a subset of a complete \(\R\)-tree, the closure of the union of
geodesics joining its points is its tight span
\citep[Lem.~2.3]{basso2023}. Apply this to the embedded \(\Qp\).
By the definition of \(\Gamma_p\) in \Cref{sec:parameter-space}, this
union is already all of \(\Gamma_p\). The required identification
follows, and uniqueness fixing the embedded \(\Qp\) is the standard
uniqueness of tight spans \citep[Sec.~3]{lang2013}.
\end{proof}

Descent selects among the directions leaving a point.

\begin{definition}[Directions and metric tangent cone]
\label{def:metric-tangent-cone}
For \(x\in\Gamma_p\), let
\(\Sigma_x\Gamma_p:=\pi_0(\Gamma_p\setminus\{x\})\)
be the set of incident tangent directions in \(\Gamma_p\), where
\(\pi_0\) denotes connected components. A direction is represented by a unit-speed geodesic germ \(\gamma : [0, t_0) \to \Gamma_p\) with \(\gamma(0) = x\), which we identify with the ray it represents. At a point \(a \in K\), this set has one
member, \(\gamma(t)=\zeta_{a,t}\). At a vertex it has \(p+1\) members, the parent and
the \(p\) child directions.
Changing center or scale representatives may relabel the children but
not the underlying ray; \cref{eq:unit-paths} fixes them so that the label is a digit.

The \textit{metric tangent cone} \(T_x\Gamma_p\) consists of velocities
\((q,s)\) with \(q\in\Sigma_x\Gamma_p\) and \(s\geq0\),
with all zero-speed pairs identified. Its metric is
\[
d_T\bigl((q,s),(q',t)\bigr)
=
\begin{cases}
|s-t|, & q=q',\\
s+t, & q\ne q'.
\end{cases}
\]
\end{definition}

Reparametrizing gives \(\vgamma_{c\vv}(h)=\vgamma_{\vv}(ch)\) for \(c\ge0\), so
\(D_{\vv}L\) is positively homogeneous: \(D_{c\vv}L=c\,D_{\vv}L\). Hence
\cref{eq:gd} may be evaluated at any positive multiple of \(\vv^*\) with the
time rescaled accordingly.

\paragraph{Coupled parameter groups for affine models.}
\label{app:directional-groups}
Fix the parameters and minibatch. Index the scalar components of the model's output
on the minibatch examples by \(a\in\mathcal R\). Write
\(F_a(u)=b_a+\sum_j a_{aj}u_j\), combining repeated occurrences
of each parameter exactly. For
\(\theta_j=\zeta_{c_j,r_j}\), define
\[
R_a=\max\!\left(\{R_a^0\}\cup
 \{|a_{aj}|_p r_j:j\in\mathcal J_a\}\right),
\qquad
A_a=\{j\in\mathcal J_a:|a_{aj}|_p r_j=R_a\},
\]
where \(\mathcal J_a\) indexes the nonzero coefficients and
\(R_a^0\) is a fixed radius contribution, zero when absent.
The active set \(A_a\) retains every tie, including parameters
that will remain stationary during the update.

Let \(\vz_a=\varphi_a(F_a(\vtheta))\) be the vector of real quantities
computed from output \(a\) for use in the loss. Write
\(L=\ell((\vz_a)_{a\in\mathcal R})\), with \(\ell\) differentiable at
the current values, and suppose the directional derivatives of
\(\vz_a\) exist. The chain rule gives
\[
D_{\vv}L=\sum_a c_a(\vv),\qquad
c_a(\vv)=\left\langle\nabla_{\vz_a}\ell,D_{\vv}\vz_a\right\rangle.
\]

A parameter outside \(A_a\) has a radius contribution strictly below
\(R_a\). For sufficiently short motion with fixed finite speeds,
that contribution remains below the maximum. At a parameter vertex,
its center shift also has norm below \(R_a\), so the output germ is
unchanged. Thus \(D_{\vv}\vz_a\), and hence \(c_a(\vv)\), depends only on
the velocities of parameters in \(A_a\).

In the affine implementation, collect identical functions to write
\(D_{\vv}L=\sum_h\alpha_h h(\vv)\).
Functions are identical only if they agree for every allowed joint velocity,
including every child direction at a vertex.
For each \(\alpha_h\ne0\), group the parameters on which \(h\) depends,
merging overlapping groups. Summing terms within each group gives the
decomposition below.

\subsection{Coordinate restriction}
\label{app:coordinate-restriction}
Fix the parameter state and minibatch. Suppose a partition
\(\mathcal G\) satisfies
\[
D_{\vv}L=\sum_{G\in\mathcal G}\psi_G(\vv_G).
\]
Let \(\mathcal V\) be the product of the coordinate velocity sets
allowed by the parameter domain. Select one coordinate from each
group, forming a set \(S\), and put
\(\mathcal V_S=\{\vv\in\mathcal V:s_i=0\text{ for }i\notin S\}\).
Write \(\mu_G\) for the least slope over unit velocities supported on group
\(G\)'s selected coordinate, a minimum over its incident directions. For unit
\(\vv\in\mathcal V_S\), put \(w_G=\|\vv_G\|_2\); homogeneity gives
\(D_{\vv}L=\sum_{G:w_G>0} w_G\,\psi_G(\vv_G/w_G)\ge\langle w,\mu\rangle\).
If every \(\mu_G\ge0\), every such derivative is nonnegative, so
\cref{eq:gd} requests non-positive time and no coordinate moves.
Otherwise, Cauchy--Schwarz over \(\{w\ge0,\|w\|_2=1\}\) gives the
minimum at \(w^\star_G\propto\max(0,-\mu_G)\), attained by choosing
a minimizing direction for each selected coordinate with positive weight.
These directions use the original coordinate derivatives, and a group
with \(\mu_G\ge0\) receives zero speed.

\paragraph{Equivariance.}
Under an isometry, unit-speed paths and their directional derivatives are
preserved, provided the objective is transformed accordingly. The
following proposition concerns the unrestricted gradient rule.

\begin{proposition}[Isometric equivariance of gradient descent]
\label{prop:equivariance}
Let \(g\) be a surjective isometry of
\((\Gamma_p^d,d_{\text{arc}})\) and define \(L^g=L\circ g^{-1}\).
For every unit velocity \(\vv\) at \(x\) for which \(D_{\vv}L(x)\) exists,
\(D_{g_*\vv}L^g(gx)=D_{\vv}L(x)\).
Thus \(g\) preserves the minimizing directions and assigned
speeds in \Cref{eq:gd}.
\end{proposition}

\begin{proof}
The image under \(g\) of a unit-speed geodesic germ is again a
unit-speed geodesic germ, and \(L^g(g\vgamma(t))=L(\vgamma(t))\).
The difference quotients, minimizing directional derivatives, and assigned speeds
therefore agree.
\end{proof}

For grouped descent, assume the transformation carries each group,
each coordinate's allowed directions, and its next vertices to
their counterparts. Each coordinate retains whether it is awaiting a
turn. Uniform selection among awaiting coordinates and among minimizing
directions then gives corresponding velocities in distribution.
Isometries preserve distances to the next vertices, so the same
requested time gives the same clipping time. The scheduling rule updates
the awaiting-turn flags identically in the two runs.

An isometry acts on the optimizer state by mapping stored
directions and carrying each coordinate's scalar moments
and update counter with that coordinate.
We denote quantities for the transformed objective
\(L^g=L\circ g^{-1}\) by primes: \(x'_t\), \(m'_t\), \(v'_t\),
and \(s'_t\).

\begin{proposition}[Equivariance of Momentum and Adam]
\label{prop:adaptive-equivariance}
The coordinate Momentum and Adam update formulas of
\Cref{sec:optimizers,app:adaptive-updates} are equivariant
in distribution under factorwise tree isometries and
coordinate permutations, acting on the objective and
optimizer state as above.
\end{proposition}

\begin{proof}
Choose corresponding scheduled coordinates and tied
directions in the two runs. Their uniform distributions
are preserved by the isometry and coordinate permutation.
Write \(g_t(\gamma)\) and \(g'_t(\gamma')\) for the current
directional derivatives of the coordinate loss in the two runs.
An isometry maps the upward direction to the upward direction
and the children of a vertex to the children of its image. Together with
\Cref{prop:equivariance}, this gives
\[
g'_t(g_*\gamma)=g_t(\gamma),
\]
so the downward derivatives and the children's excesses correspond.
By induction on the updates, so do \(u_t\), \(d_t\), and the averages
stored at corresponding vertices. Therefore \(m'_t(g_*\gamma)=m_t(\gamma)\),
and the condition for moving into a child holds for \(\gamma\) exactly
when it holds for \(g_*\gamma\).
For Adam, the minimum current directional derivative, stored second moment,
and update counter agree, so \(v'_t=v_t\) and the
bias corrections agree. Thus both methods satisfy
\[
s'_t(g_*\gamma)=s_t(\gamma).
\]
Corresponding tie choices select corresponding paths with
equal requested distances. Isometries preserve vertices
and distances, so clipping gives \(x'_{t+1}=g(x_{t+1})\).
The stored averages and counters continue to
correspond after either a move or a stay.
\end{proof}

The result extends to grouped Momentum and Adam under the group, direction, and boundary conditions above. Select coordinates and tied directions in the original and transformed runs that map to each other under the transformation. The preceding proof applies to each selected coordinate, while unselected coordinates retain their state. With the same learning rate \(\alpha_t\), equal speeds and distances to the next vertices give the same step size after clipping.

For \(d_{\text{arc}}\), this applies to translations, unit affine scalings, their
factorwise product actions, and coordinate permutations. %

\section{Evaluation and Extension on Product States}
\label{app:product-eval}

For analytic functions, we assume convergence of the centered
series in the product Gauss seminorm: its weighted terms
\(|a_\alpha|_p\prod_j r_j^{\alpha_j}\) tend to zero as
\(|\alpha|\to\infty\). Zero-radius coordinates are first specialized
to their centers.

In the notation of \Cref{sec:forward_propagation}, the centered Taylor expansion is%
\begin{equation}
\label{eq:taylor-series}
F_{k,i}(\vz) \;= F_{k,i}(\vc) + \sum_{|\vI| \ge 1} a_{\vI} \prod_{j=1}^{d_{k-1}} (z_j - c_j)^{I_j}
= F_{k,i}(\vc) + a_{1,\dotsc,0}(z_1-c_1) + \dotsb.
\end{equation}

The product Gauss seminorm of \Cref{rem:product-points} is independent of the chosen center representatives: apply \Cref{lem:representatives} in each coordinate. The weighted-maximum formula is symmetric in the coordinates, so the order in which they are evaluated does not matter.%

\begin{corollary}[Representative invariance of tangent data]
\label{rem:rep-invariance}
Equivalent representatives give the same output state and the same
loss derivative along each geometric ray.
\end{corollary}

\begin{proof}
The preceding argument and \Cref{cor:pushforward} give the same output state for equivalent representatives. The loss along a fixed geometric ray
is therefore unchanged, as is its directional derivative.
\end{proof}

Changing representatives may relabel child directions. The centered
expansion, including its active-term list, can change under recentering.

\begin{remark}[\(K\)-points are not dense]
\label{rem:nondense}
In the closed unit disk over \(\mathbb Q_p\), the Gauss value of
\(T^p-T\) at \(\zeta_{0,1}\) is \(1\), whereas
\(|a^p-a|_p\le p^{-1}\) for every \(a\in\mathbb Z_p\). The \(K\)-points
are therefore not dense in this disk, and the extension in
\Cref{prop:functoriality} is determined by the bounded
algebra homomorphism rather than by any density argument.
\end{remark}

\begin{proposition}[Functoriality of affinoid spectra]
\label{prop:functoriality}
Let \(K\) be a complete non-Archimedean field, let \(A\) and \(B\) be
\(K\)-affinoid algebras, and let \(\varphi:B\to A\) be a bounded
\(K\)-algebra homomorphism. Then \(\varphi\) determines a unique
morphism of \(K\)-affinoid analytic spaces
\[
\mathcal M(A)\longrightarrow\mathcal M(B),
\qquad
x\longmapsto\bigl(b\mapsto|\varphi(b)|_x\bigr).
\]
On \(K\)-rational points, this morphism agrees with the map represented
by \(\varphi\).
\end{proposition}

\begin{proof}
For \(x\in\mathcal M(A)\), the rule \(b\mapsto|\varphi(b)|_x\) is a
bounded multiplicative seminorm on \(B\) that extends the norm on \(K\),
and hence defines a point of \(\mathcal M(B)\). Functoriality of the
Berkovich spectrum makes this map analytic. When \(x\) is evaluation at
a \(K\)-rational point, the same formula is evaluation after the
original affinoid map. Uniqueness follows because morphisms of affinoid
spectra are contravariantly determined by their bounded algebra
homomorphisms, not by density of \(K\)-points
(\Cref{rem:nondense}).
\end{proof}

\begin{corollary}[Pushforward of a product state]
\label{cor:pushforward}
Let \(x=(\zeta_{c_j,r_j})_{j=1}^{\nparams}\) and let \(f\) have centered
expansion \(f(T)=f(c)+\sum_{|\alpha|\ge1}a_\alpha (T-c)^\alpha\), with
\(R=\max_{|\alpha|\ge1}|a_\alpha|_p\prod_j r_j^{\alpha_j}\). The image of \(x\) under the morphism of
\Cref{prop:functoriality} induced by \(g\mapsto g\circ f\) is
\(\zeta_{f(c),R}\); that is, \(|g\circ f|_x=|g|_{f(c),R}\) for every
polynomial \(g\) in one variable. For vector-valued \(f\) the statement
applies to each output coordinate, which is the display
\cref{eq:pushforward} that the forward evaluation computes output by
output.
\end{corollary}

\begin{proof}
Write \(h=f-f(c)\), so \(|h|_x=R\), and expand
\(g(S)=\sum_{n\ge0}q_n(S-f(c))^n\), so that \(g\circ f=\sum_n q_nh^n\) and
\(|g|_{f(c),R}=\max_n|q_n|_pR^n=:M\). If \(R=0\), then \(|h|_x=0\), so every positive power of \(h\)
has seminorm zero and both sides equal \(|q_0|_p\).
Now assume \(R>0\) and \(g\neq0\). The strong
triangle inequality and \(|h^n|_x\le R^n\) give \(|g\circ f|_x\le M\).

For the reverse bound, call \(|t|_p\prod_j r_j^{\gamma_j}\) the weight of
a term \(t\,(T-c)^\gamma\); \(|\cdot|_x\) is the largest weight among the
terms of a collected expression. Order exponents lexicographically; this
order is total, compatible with addition, and has \(0\) as least
element. Let \(\beta\) be the largest exponent whose term in \(h\) has
weight \(R\); then \(\beta\neq0\). Let \(n^\ast\) be the largest \(n\)
with \(|q_n|_pR^n=M\). In \(h^n\), the coefficient of
\((T-c)^{n^\ast\beta}\) is the sum of \(\prod_i a_{\beta_i}\) over
\(n\)-tuples with \(\sum_i\beta_i=n^\ast\beta\). Each summand has weight
at most \(R^n\), with equality only when every \(\beta_i\) has weight
\(R\), hence \(\beta_i\preceq\beta\) and \(\sum_i\beta_i\preceq n\beta\),
with equality only when every \(\beta_i=\beta\). For \(n=n^\ast\) exactly
one summand, \(a_\beta^{n^\ast}\), has weight \(R^{n^\ast}\). For
\(n<n^\ast\) none does, since \(n\beta\prec n^\ast\beta\). For
\(n>n^\ast\), \(|q_n|_pR^n<M\). So in \(\sum_nq_nh^n\) the coefficient of
\((T-c)^{n^\ast\beta}\) is one term of weight \(M\) plus terms of smaller
weight, and therefore has weight \(M\). Hence \(|g\circ f|_x\ge M\).
\end{proof}

For vector-valued \(f\) the coordinate images need not determine its full analytic image (\Cref{ex:coordinate-marginals}).

\section{Automatic Differentiation on the \(p\)-adic Hull}
\label{app:phad}

For the chain rule in \Cref{sec:backward-propagation}, write
\(F=F_K\circ\cdots\circ F_1\) for the forward map at a fixed input, with
\(F_k:X_{k-1}\to X_k\), \(X_k=\Gamma_p^{d_k}\), and loss head \(\ell:X_K\to\R\).
Only the parameters are extended to the hull; the input stays in \(\Q_p^{\nin}\).
We assume that \(F_K\circ\cdots\circ F_1\) computes the forward map of
\Cref{sec:forward_propagation} exactly. Intermediate outputs that depend
on the same parameter cannot, in general, be replaced by independent
coordinate disks: doing so can change subsequent evaluations
(\Cref{ex:coordinate-marginals}).
Write \(x_0=\vtheta\) and \(x_k=F_k(x_{k-1})\).
Equip each \(X_k\) with the product metric \(d_{\text{arc}}\) of
\Cref{eq:product-metric}.
When the output speed is positive, the tangent map
\(T_k:=T_{x_{k-1}}F_k\) sends the input unit velocity \(\vv_{k-1}\)
to \(\vv_k=T_k(\vv_{k-1})\).
Each \(\vgamma_{\vv_k}\) is parametrized by product arclength and starts at \(x_k\).
The one-sided metric speed of \(F_k\) along \(\vv_{k-1}\) is
\[
\lambda_k^{\mathrm{arc}}(x_{k-1},\vv_{k-1})
=\lim_{t\downarrow0}
\frac{d_{\text{arc}}\bigl(F_k(\vgamma_{\vv_{k-1}}(t)),F_k(x_{k-1})\bigr)}{t}.
\]
Let \(H_k=\ell\circ F_K\circ\cdots\circ F_{k+1}\) be the loss
remaining after \(F_k\), with \(H_K=\ell\) and \(H_0=\ell\circ F\).
\begin{proposition}[Directional chain rule]
\label{prop:directional-chain-rule}
Fix an input velocity \(\vv_{k-1}\) with finite speed
\(\lambda_k^{\mathrm{arc}}\).
If \(\lambda_k^{\mathrm{arc}}>0\) and \(D_{\vv_k}H_k\) exists and is
finite, the backward identity
\(D_{\vv_{k-1}}H_{k-1}=\lambda_k^{\mathrm{arc}}D_{\vv_k}H_k\)
holds if and only if
\[
H_k\!\left(F_k(\vgamma_{\vv_{k-1}}(h))\right)
-H_k\!\left(\vgamma_{\vv_k}(\lambda_k^{\mathrm{arc}}h)\right)=o(h).
\]
When \(\lambda_k^{\mathrm{arc}}=0\), the rule
\(D_{\vv_{k-1}}H_{k-1}=0\) holds if and only if
\[
H_k\!\left(F_k(\vgamma_{\vv_{k-1}}(h))\right)-H_k(x_k)=o(h).
\]

\end{proposition}

These conditions say that the tangent approximation preserves
the first-order change in the remaining loss.
These loss estimates hold if \(H_k\) is locally Lipschitz near
\(x_k\) and, at positive speed,
\[
d_{\text{arc}}\bigl(
F_k(\vgamma_{\vv_{k-1}}(h)),\vgamma_{\vv_k}(\lambda_k^{\mathrm{arc}}h)
\bigr)=o(h).
\]
Local Lipschitz continuity of the remaining maps and \(\ell\)
ensures that of \(H_k\).
Applying the identities at each step gives the backward
recurrence in \Cref{sec:backward-propagation}.

\begin{proof}
For positive speed, the directional derivative gives
\(H_k(\vgamma_{\vv_k}(\lambda_k^{\mathrm{arc}}h))-H_k(x_k)
=\lambda_k^{\mathrm{arc}}hD_{\vv_k}H_k+o(h)\).
Using \(H_{k-1}=H_k\circ F_k\), subtracting this expansion
and dividing by \(h\) proves the first equivalence.
The second equivalence is the definition of a zero
directional derivative.

For the sufficient check, a local Lipschitz constant \(C_k\) gives
\(|H_k(u)-H_k(v)|\le C_k d_{\text{arc}}(u,v)\).
At positive speed, applying this bound to the actual and
approximating paths gives the first loss estimate.
At zero speed, the speed definition gives
\(d_{\text{arc}}(F_k(\vgamma_{\vv_{k-1}}(h)),x_k)=o(h)\),
so taking \(v=x_k\) gives the second loss estimate.
The remaining maps and \(\ell\) form a finite composition,
so their local Lipschitz bounds also give one for \(H_k\).

If all speeds are positive, start from
\(g_K(\vv_K)=D_{\vv_K}\ell\).
The identities give
\(g_{k-1}(\vv_{k-1})=D_{\vv_{k-1}}H_{k-1}\)
by backward induction, ending at \(g_0(\vv_0)=D_{\vv_0}L\).
A zero-speed step supplies the derivative zero directly,
and the same induction applies to the preceding steps.
\end{proof}

In the stage notation of \Cref{sec:backward-propagation}, the resulting velocity \(\hat{\vv}=(\hat{\vq},\hat{\vs})\) has speeds \(\hat s_i=|D_{\vv}\hat r_i|\) and directions \(\hat q_i\): up if \(D_{\vv}\hat r_i>0\), and down if at an edge and \(D_{\vv}\hat r_i<0\), or (if at an inner vertex) down via digit \(a_{\hat n}\) of \(F_{k,i}(\vc')\), where \(\hat r_i=p^{-\hat n}\) and \(\vc'\) is \(\vc\) with \(c_j\) replaced by \(c_j+\rho_jp^{n_j}\) in every coordinate moving down via a digit \(\rho_j\).

At an inner vertex, apply this derivative calculation to each incident input direction separately.

For a velocity moving several coordinates, the rate is taken jointly
(\cref{eq:chain-rule}).
\begin{example}[A tie couples the coordinates]
\label{ex:tie-couples}
Take \(f(\vtheta)=x_1\theta_1+x_2\theta_2\) with \(|x_1|_p=|x_2|_p=1\) and both
parameters at radius \(r\), so \(R=r\) and both monomials attain the maximum.
Moving coordinate \(1\) up alone gives \(\dot R=1\), but moving it down alone
gives \(\dot R=0\), since the second monomial then attains the maximum. Moving
both up at unit rate gives \(\dot R=1\), not \(2\): the single-coordinate rates do not add.
\end{example}

\section{Sparse Affine Backpropagation}
\label{app:coincidence}

\begin{lemma}[Direct loss along a tree germ]
\label{lem:direct-tree-germ}
Let \(T\) be an \(\mathbb R\)-tree, let \(y\in T\), and define
\(\ell_y(x)=d_T(x,y)/2\). For a unit-speed incident germ \(\gamma\) at \(x\),
\[
D_\gamma^+\ell_y(x)=
\begin{cases}
-\frac12,&x\ne y\text{ and }\gamma\text{ points toward }y,\\[3pt]
+\frac12,&\text{otherwise}.
\end{cases}
\]
At \(x=y\), every nonconstant germ has derivative \(+1/2\).
\end{lemma}

\begin{proof}
For sufficiently small \(t>0\), motion toward \(y\) shortens the unique
path to \(y\) by \(t\), while every other germ lengthens it by \(t\).
\end{proof}

For an affine output whose term in \(\theta_j\) strictly attains the output radius, the \(p\) children of
\(\theta_j\) map to the \(p\) distinct children of the output, since multiplication by a unit permutes the
residues; otherwise no child of \(\theta_j\) moves the output. By \Cref{lem:direct-tree-germ}, at most one
child per output then has a slope different from the rest.
The classification losses depend on each output \(\zeta_{\hat c,\hat r}\) only through
\(\max\{|\hat c|_p,\hat r\}=d_{\mathrm{arc}}(\zeta_{\hat c,\hat r},0)/2+\hat r/2\) (\Cref{sec:classification});
the radius has the same derivative along every child, so the same holds with target \(0\).

Output \(a\) therefore contributes
\[
s_{a,\gamma}=b_a+c_a\mathbf1\{\gamma=\rho_a\}
\]
to the slope along child \(\gamma\), where \(\rho_a\) is the child whose slope differs, and \(c_a=0\) when
output \(a\) contributes the same slope along every child. One pass over the \(m\) contributing outputs
accumulates \(B=\sum_a b_a\) and \(C_\rho=\sum_{a:\rho_a=\rho}c_a\); the child slopes are \(B+C_\gamma\),
for \(O(m+p)\) operations in place of \(O(mp)\).

\begin{example}[The affine direct loss need not be geodesically convex]
\label{rem:affine-nonconvexity}
Take \(p=3\), \(F(\theta_1,\theta_2)=\theta_1+\theta_2\),
and target \(1\). The product path
\((\zeta_{0,1/2+t},\zeta_{0,1/2-t})\), \(|t|\le1/12\),
is a constant-speed geodesic contained in a product of open
edges. Its output center is zero and its output radius is
\(1/2+|t|\). Hence its direct loss is \(3/4-|t|/2\),
which has a strict maximum at the midpoint.
\end{example}

\section{Implementation and Experimental Details}
\label{app:numerical-rule}
\label{app:implementation_details}

\subsection{Optimizer updates}
\label{app:optimizer-updates}

\label{app:adaptive-updates}

\paragraph{Momentum state.}
The upward and downward averages satisfy
\[
\begin{aligned}
u_t&=\beta\,u_{t-1}+(1-\beta)\,D_{\ve_+}L(\vtheta_t),\\
d_t&=\beta\,d_{t-1}+(1-\beta)\,\min_{q\ne+}D_{\ve_q}L(\vtheta_t).
\end{aligned}
\]
Each coordinate stores \(u_t\) and \(d_t\) and, at every vertex it has
visited, the average excess \(\bar\Delta_t(\rho)\) and the average directional derivative along each
child \(\rho\). For a quantity \(x\) observed at some updates \(\tau\le t\) of
the coordinate,
\[
\bar x_t=(1-\beta^{\,n_t})\,\frac{\sum_\tau\beta^{\,n_t-n_\tau}x_\tau}{\sum_\tau\beta^{\,n_t-n_\tau}},
\]
where \(n_t\) is its update counter; a child's averages use only the updates
at its vertex.
Retaining each coordinate's momentum between its selections
follows the history convention in the cyclic coordinate heavy-ball update
of \citet[eq.~(23)]{sun2019heavyball}.

\paragraph{Adam.}
The second moment and speeds along the incident directions are
\begin{equation*}
\begin{aligned}
v_t &= \beta_2 v_{t-1} + (1 - \beta_2)\,\bigl(\max(0,-\min_q D_{\ve_q}L(\vtheta_t))\bigr)^2,\\
s_t(q) &= \max(0,-\hat m_t(q)) / (\sqrt{\hat v_t} + \varepsilon),
\qquad s^*_t = \max_q s_t(q).
\end{aligned}
\end{equation*}
On an edge where \(L\) is differentiable, the squared term in the \(v_t\) update is \((dL/dr)^2\).
The child-selection condition of \Cref{sec:optimizers} also applies.
The hats are the standard counter-based bias
corrections, \(\widehat m_t=m_t/(1-\beta^{\,n_t})\) and
\(\widehat v_t=v_t/(1-\beta_2^{\,n_t})\), with \(n_t\) counting successful updates of this parameter,
including those with zero speed.

\paragraph{Ties.}
When several directions share the maximal speed, one is chosen uniformly.

\paragraph{Coordinate queueing.}
Initially, every coordinate is marked as awaiting a turn. Before each
update, determine the groups from the current parameters and minibatch.
Splitting or merging groups leaves each coordinate's mark unchanged.
Within each group, select uniformly among the coordinates awaiting a
turn; if none await a turn, first mark every coordinate in that group
as awaiting. Draw independently for different groups, and complete all
selections before changing any parameter.

After successfully computing an update, mark each selected coordinate
as no longer awaiting a turn, even if it did not move. Leave all other
marks unchanged.

The random-permutation guarantee assumes fixed groups; under changing
groups, the rule does not guarantee that every coordinate receives
repeated turns.

\subsection{Numerical implementation and cost}

\paragraph{Exact arithmetic and states.}
\label{app:arithmetic-resolution}
Every ordinary center, input, target, and
internal \(p\)-adic constant used by a reported run lies in
\(\mathbb Z[1/p]\). Centers are stored exactly as a signed mantissa times a power of \(p\), with \(19\) significant digits in int32 for \(p=3\) (\(31\) for \(p=2\); \(13\) for \(p=5\)). A coordinate whose next digit would exceed this capacity stops descending; the coordinate can still move upward.

The radius is stored as \(\log_p r\), so vertices are integer-valued and clipping lands on them exactly.
Disk scores are independent of the chosen center
representative (\Cref{app:product-eval}); the point reduction need not be.

\paragraph{Grouping cost.}
\label{app:grouping-cost}
With \(\mathcal R\) and \(\mathcal J_a\) as in
\Cref{app:directional-groups}, let \(M=|\mathcal R|\) be the number of
outputs across the minibatch, \(I=\sum_a|\mathcal J_a|\) the number of
nonzero collected coefficients summed over outputs, and \(\nparams\) the
number of parameters. The bound below covers grouping after repeated coefficients have been
combined and the outer loss derivatives computed.

Assuming hash-table lookups take constant time on average, forming the
active sets, collecting contributions, and combining overlapping groups
require
\[
O(M+I+d)+Z_t
\]
expected work, with \(O(M+I+d)\) graph storage, where \(Z_t\) denotes any
additional work that the tests or comparisons require at update \(t\). When
the local tests have bounded cost, grouping takes expected linear time in
these counts. Each bound counts arithmetic operations on numbers stored with a
fixed number of digits.

\paragraph{Cost.}
\label{app:evaluation-cost}
Under the arithmetic and hash-table assumptions of \Cref{app:grouping-cost},
scalar affine regression with \(d\) coefficients and a batch of \(b\) examples
uses \(O(bd)\) expected work for the forward computation and grouping, shared by
all selected coordinates. With one coordinate selected per group, computing their
child slopes separately gives an update cost of at most \(O(bd+bp|\mathcal G|)\).
Sparse affine backpropagation (\Cref{app:coincidence}) reduces the slope
computation for all optimizers to \(O((b+p)|\mathcal G|)\), giving
\(O(bd+p|\mathcal G|)\) overall since \(|\mathcal G|\le d\). Scoring one beam
candidate also costs \(O(bd)\), so extending a beam of width \(W\) to
\(C=Wp^{d}\) candidates costs \(O(Cbd)\), plus \(O(C\log C)\) for our full sort.
For more general models, the costs of forward and backward propagation depend on
how the model is computed and need not be linear in the number of parameters.

\subsection{Experimental protocol}
\label{app:training-protocol}

\paragraph{Tuning and budgets.} All three tree optimizers use
\(\alpha=\kappa R_{\mathrm{ref}}(1-1/p)\), with reference radius \(1\) for
regression, \(64\) for Quillian, and the smallest \(p^s\ge m\) for modulo \(m\).
The multiplier \(\kappa\) is searched over \(\{0.01,0.1,1,10,100\}\) for
regression and for classification Adam, and over \(\{0.1,1,10,100,1000\}\) for
classification GD and Momentum. Momentum uses decay \(0.9\); Adam uses \(0.9\),
\(0.999\), and \(\varepsilon=10^{-8}\). Every task trains at batch size 32, for
1,000 updates on regression and modulo and 2,000 on Quillian. Regression uses
512 examples in each of its training, validation, and independent test streams;
modulo uses 2,160, 720, and 720. Quillian's inner split holds 1,076 training and
268 validation propositions, and each selected configuration is trained again on
the full 1,344-proposition training portion before its 336-proposition test
portion is evaluated.

\paragraph{Selection and reporting.} For the p-adic gradient methods,
modulo selects the grid point with the highest mean validation
accuracy over the specified splits, and Quillian selects one multiplier for
the complete model, for each code assignment, by pooled validation average
precision (\Cref{app:quillian-protocol}); in both, ties go to the smaller multiplier.
Point results use the same selected fit.
Regression selects the multiplier separately for each coefficient instance,
optimizer, and initialization by the work to first satisfy
\(\max_i|\hat f(\vx_{\mathrm{val}}^{(i)})-y_{\mathrm{val}}^{(i)}|_3\le3^{-5}\)
on all 512 validation examples. Exact ties go to the smaller multiplier;
if no multiplier reaches this threshold, the same tie rule applies.
Beam width is selected from five and ten by the most validation-recovered
minibatch streams, then the least mean first-recovery work among them,
then the smaller width.
Learning rates and beam widths are tuned only on validation data.
Coefficient recovery is reported for the selected
fits; final test losses include all five instances, including nonrecoveries.
Reported training work excludes the rate and width searches and validation monitoring.
Tables report the mean and sample standard deviation over the reporting
seeds.
Regression validation uses 512 independent covariate and noise draws
conditional on the same \(\vtheta^*\) as training. Rates are
selected separately for each optimizer, initialization and batch size.
The classification splits overlap, so some test observations enter shared
rate selection through another split's validation set.

\paragraph{Euclidean optimization.}\label{app:euclidean-optimization}
The Quillian MLP uses
full-data Adam at Martins's learning rate \(0.1\) for 1,000 updates.
Euclidean modulo models use Adam with batch size 32 for
6,750 updates, totaling 216,000 sampled training examples per fit,
with repetitions counted separately.
Learning rates \(\{0.003,0.01,0.03,0.1,0.3\}\) are selected by
five-seed validation accuracy.

\paragraph{Setup and initialization.}
For each candidate input coordinate and depth through the
supplied maximum, initialization tests whether the positive training
examples in each relation context occupy one residue class. Among
compatible choices it minimizes the number of negative examples in
those residue classes, then chooses the coarsest depth and the first
coordinate in the fixed input ordering.
At selected residue depth
\(v\), the chosen coefficient starts at \(\zeta_{p^{-v},1}\) and the other
coefficients at \(\zeta_{0,p^v}\).
Quillian tests its three entity coordinates, grouped by relation, at
depths zero through six. For modulo \(m\) with model prime \(q\), the scalar
input coordinate is tested at depths zero through the smallest integer
\(D\ge0\) satisfying \(q^D\ge m\), separately for each class.
Euclidean modulo models use LeCun-normal weights, zero biases and no weight
decay. The Quillian MLP instead
initializes its embedding, hidden and output weight matrices with
independent zero-mean Gaussian entries of standard deviation \(0.1\),
and its output biases at zero. The hidden layer's bias is the weight of
the constant input, so it is initialized like the other weights.
Beam search uses the same training labels to rank and prune candidates,
without this initialization (\Cref{app:quillian-beam}).

\paragraph{Extracting \(p\)-adic parameters.}\label{app:classification-point-evaluation}
For classification point evaluation we select \(c_0+p^N\) from each
positive-radius coefficient disk \(\zeta_{c,r}\), where
\(N=-\lfloor\log_p r\rfloor\) and \(c_0\) retains only the base-\(p\) digits of
\(c\) at exponents below \(N\). A zero binary seminorm gives probability
one for presence; for multiclass scores, probability is distributed
equally among exact-zero scores when any occur. Ties among maximizing
classes are broken by the fixed class ordering.

\subsection{Regression}
\label{app:regression-protocol}

\paragraph{Adverse initialization.}
Write coefficients in base \(3\) with the least significant digit last, as in
\Cref{sec:background}. For the first seed the true coefficients are
\[
\vtheta^*=(\dotsc002012020_3,\ \dotsc022001202_3,\ \dotsc011000122_3)
=(1599,\ 5879,\ 2933).
\]
Setting the last digit of the first coefficient to \(1\) and holding it fixed,
strict best-improvement over the remaining \(23\) digit positions stops at
\[
(\dotsc022012021_3,\ \dotsc022221201_3,\ \dotsc011000102_3)
=(5974,\ 6526,\ 2927),
\]
where all \(48\) single-digit changes weakly increase the exact training loss, the smallest
increase being \(1/2519424\). Datasets are redrawn until such a point exists;
for this seed the first draw was rejected. The accepted dataset is also
used for the zero-center runs, and its covariates and noise are reused in
the two-layer task.

\paragraph{Two-layer target.}
\label{app:two-layer-target}
We draw the target coefficients \(w_1,w_2,v_1,v_2\) uniformly from
\(\{0,\ldots,3^8-1\}\), and draw the fourth input independently.
Gradient training and beam search use
\(f(\vx)=v_1(x_1+w_1x_2)^2+v_2(x_3+w_2x_4)^2\),
with all four gradient parameters initialized at the Gauss point.
Stagewise propagation is exact for this model because the two terms use disjoint parameter sets.
To check recovery, we expand the learned and target functions
and compare the coefficients of every \(x_i^2\) and \(x_ix_j\), \(i<j\).

We also tested the ten-parameter model
\(F(\vx)=\sum_{k=1}^2v_k(\sum_{j=1}^4a_{kj}x_j)^2\).
Multiplying a hidden unit by \(\lambda\neq0\) and dividing its output
weight by \(\lambda^2\) leaves \(F\) unchanged.
We initialized \(a_{11},a_{23}\) at \(\zeta_{1,3^{-9}}\),
\(a_{13},a_{14},a_{21},a_{22}\) at \(\zeta_{0,3^{-9}}\),
and the remaining four parameters at the Gauss point, with all ten trainable.
Replacing either the two states \(\zeta_{1,3^{-9}}\) or the four states
\(\zeta_{0,3^{-9}}\) by Gauss points reached depth-five recovery on none
of five seeds at \(\kappa\in\{1,10,100\}\) within 1,000 GD updates.

On the seed missed at width five, the target \(v_1\) is divisible by \(9\).
After two digits, every candidate with \(v_1=0\) gives the same loss whatever
digits \(w_1\) has, so nine tie exactly and the beam keeps five. At the
third digit these nine take the nine lowest losses: they agree with the
target modulo \(27\) in \(w_2\), \(v_1\) and \(v_2\), and differ only in
\(w_1\), with the candidate matching the target ninth. Mean point \(L_1\)
loss on the full data:
\begin{center}
\small
\setlength{\tabcolsep}{2pt}
\begin{tabular}{r|ccccccccc|c}
rank & 1 & 2 & 3 & 4 & 5 & 6 & 7 & 8 & 9 & 10\\
\(w_1\bmod 27\) & 24 & 6 & 15 & 0 & 9 & 18 & 12 & 3 & 21 & \\
loss & 0.02692 & 0.02711 & 0.02712 & 0.02751 & 0.02760 & 0.02765 & 0.03295 & 0.03296 & 0.03297 & 0.06045\\
\end{tabular}
\end{center}
The tenth candidate is the first that differs from the target in \(v_1\) and
\(v_2\). Full-data
scoring fails on the same seed as minibatch scoring and costs sixteen
times as much.

\paragraph{Regression beam search.}
\label{app:regression-beam}
We search depths \(v=0,\ldots,7\), indexed by the exponent in \(r=p^{-v}\),
using mean point \(L_1\) loss, with all candidates at a depth sharing one batch.
After scoring the complete depth, we retain the best distinct candidates up to
the width, and the search returns the lowest-loss candidate at the last depth.
We average recovery work and \(\log_3\) test loss over five minibatch runs
per data seed before computing the five-seed mean and sample SD.

\subsection{Modulo classification}
\label{app:modulo-protocol}

\paragraph{Learned modulo.} We stratify by residue modulo \(36=4\cdot9\)
so that the same splits are balanced for both modulo tasks.
For each seed, we randomly partition the
integers \(0,\ldots,3599\), assigning 60, 20 and 20 members of each
residue class to training, validation and test, respectively.
All methods in the modulo comparison share these splits. For the permuted-code control, a fixed random permutation reassigns the input integers while their labels remain unchanged. We repeat the initialization procedure above using the permuted training data. The Euclidean digit baseline
concatenates one-hot encodings of the eight least significant base-\(p\)
digits, with \(p=2\) for modulo 4 and \(p=3\) for modulo 9.
Every compared method is tuned independently by the shared selection rule of
\Cref{app:training-protocol}.
At the selected multipliers, plain descent, Momentum, and Adam all reach test
accuracy \(1.000\) on the matched tasks, whether evaluating the learned parameter
states or the points chosen by the fixed extraction rule.

\subsection{Quillian benchmark}
\label{app:quillian-protocol}

\paragraph{Euclidean MLP architecture.} Figure~5 of
\citet{martins2025learning} shows a learned six-dimensional embedding for each
of the 15 entities. We concatenate that embedding with the four relation
indicators and a fixed \(1\), feed the resulting 11 inputs to 15 sigmoid hidden
units, and use an affine map to 28 sigmoid outputs. Thus the conventional
parameter count is
\[
15\cdot 6 + (6+4+1)\cdot 15 + (15+1)\cdot 28 = 703.
\]
Martins reports 690 parameters, but the count in the paper source adds 15,
rather than 28, for the final-layer bias.
The
paper specifies Adam at learning rate \(0.1\), but not the batch size or epoch
budget. Our reproduction uses the Adam protocol in
\Cref{app:euclidean-optimization}. Our five
reported seeds vary both the random split and initialization; Martins instead
holds one random split fixed across ten initializations.

\paragraph{Metrics.} Within each test split, we pool the 336 propositions
across all 28 properties and report positive-class F1, average precision (AP),
and accuracy on the pooled counts, grouping equal scores at each threshold of
the precision--recall curve.

\paragraph{Quillian permutation control.} We randomly reassign complete three-dimensional code vectors among entities, keeping the code multiset, relation indicators and proposition labels fixed. We repeat the initialization procedure above using the permuted training data.

\paragraph{Quillian beam search.}\label{app:quillian-beam} We use maximum depth 8 to cover the
deepest levels of the encoded hierarchy. Beam width 10 gives zero training
error on two of our five splits, as \citet{martins2025learning} observed
on their single split. If the search reaches the depth limit or exhausts
its admissible candidates before finding a classifier with zero training
error, we return the candidate with the fewest training errors from
its last nonempty beam. Across the five splits and 28 properties,
137 of 140 fits attain zero training error.

The batch-32 beam variant uses the same width, depth limit, and splits, with a
fresh uniform sample of up to 32 training propositions shared by all candidates
at each depth. Admissible candidates, those with no error on a positive training
example, are ranked by batch errors; a candidate with no batch error is then
checked on the remaining training negatives, in rank order and including
candidates beyond the beam cutoff, and the search stops at the first with no
full-training error.
This variant attains zero training error in 133 of 140 fits.

\section{Logit Separation and Margins}

Recall from \Cref{sec:classification} that the output for class \(k\) is
\(\hat y_k=\zeta_{\hat c_k,\hat r_k}\), with logit
\[
z_k=-\log_p\max\{\hat r_k,|\hat c_k|_p\}.
\]

\begin{theorem}[Logit margins at every radius]\label{thm:margins}
Let \(X\) be a finite training set and let the class maps have the form
\(f_k(\eta,b;\vx)=G_k(\eta;\vx)+b_k\), with independent additive biases
\(b_k\). Suppose each \(G_k(\cdot;\vx)\), for \(\vx\in X\), is analytic near
parameter centers \(\eta_0\). Assume \(X\) is \textit{\(p\)-adically separable} at centers \((\eta_0,b_0)\):
every training example \(i\) of class \(c\) satisfies
\[
|\hat c_c^{(i)}|_p\le1,\qquad
|\hat c_k^{(i)}|_p\ge p\quad\text{for every }k\ne c.
\]
Then for every \(\rho\in(0,1]\) there is a parameter state with these centers
and all parameter radii positive at which every output on \(X\) has radius
\(\rho\), and every correct-class logit exceeds every incorrect-class logit
by at least \(1\) on every training example.
\end{theorem}

\begin{proof}[Proof of \Cref{thm:margins}]
Fix \(\rho\in(0,1]\) and the separating centers. Each
\(G_k(\cdot;\vx^{(i)})\) has a centered power series that converges
for sufficiently small positive parameter radii. As there are finitely
many such series, we can choose positive radii \(s_j\) small enough that
all the expansions
\[
G_k(\eta;\vx^{(i)})
=G_k(\eta_0;\vx^{(i)})+
\sum_{|\alpha|\ge1}a_{k,i,\alpha}(\eta-\eta_0)^\alpha
\]
converge in the product Gauss seminorm of \Cref{app:product-eval}.
The weighted coefficients
\(|a_{k,i,\alpha}|_p\prod_j s_j^{\alpha_j}\), for \(|\alpha|\ge1\),
are bounded by some \(M>0\). Choose \(0<t\le1\) with \(tM\le\rho\)
and give coordinate \(\eta_j\) radius \(r_j=ts_j\). By
\Cref{cor:pushforward}, each \(G_k\) then has output radius
\[
R_k^{(i)}
=\max_{|\alpha|\ge1}|a_{k,i,\alpha}|_p\prod_j r_j^{\alpha_j}
\le tM\le\rho.
\]
Give each bias \(b_k\) radius \(\rho\). Since \(G_k\) does not depend on \(b_k\),
the centered bias term has coefficient \(1\) and cannot cancel with
any term of \(G_k\). The Gauss formula therefore gives
\(\hat r_k^{(i)}=\max\{R_k^{(i)},\rho\}=\rho\).
All parameter radii are positive, and the output centers are unchanged.
Since \(\rho\le1\), the separation inequalities give
\(z_c^{(i)}\ge0\) and \(z_{k'}^{(i)}\le-1\) for every \(k'\ne c\).
Thus \(z_c^{(i)}-z_{k'}^{(i)}\ge1\) on every training example.
\end{proof}

\end{document}